\documentclass{article} % For LaTeX2e
\usepackage{iclr2027_conference,times}

\usepackage{amsmath,amsfonts,bm}

\def\eqref#1{equation~\ref{#1}}
\def\Eqref#1{Equation~\ref{#1}}
\def\1{\bm{1}}

\DeclareMathAlphabet{\mathsfit}{\encodingdefault}{\sfdefault}{m}{sl}
\SetMathAlphabet{\mathsfit}{bold}{\encodingdefault}{\sfdefault}{bx}{n}

\usepackage{hyperref}
\usepackage{url}
\usepackage{amsthm}

\usepackage{algorithm}
\usepackage{algpseudocode}
\usepackage{graphicx}
\theoremstyle{remark}

\usepackage{booktabs}
\usepackage{subcaption}
\usepackage{caption}
\usepackage{wrapfig}

\newif\ifshowresolved
\showresolvedtrue % or \showcommentsfalse
\ifshowresolved
    \newcommand{\resolved}[1]{\textcolor{gray}{[sd-resolved: #1]}}
    
    \newcommand{\erase}[1]{\textcolor{black}{\sout{#1}}}
\else
    \newcommand{\resolved}[1]{}
    \newcommand{\erase}[1]{}
    
\fi

\title{Mind the Residual Gap: \\Probabilistic Downscaling under Real-World Bias}

\author{
Yujin Kim \\
Department of Computer Science\\
Cornell University\\
Ithaca, NY 14853, USA\\
\texttt{yk826@cornell.edu}
\And
Nidhi Soma \\
Department of Computer Science\\
Cornell University\\
Ithaca, NY 14853, USA\\
\texttt{ns848@cornell.edu}
\AND
Sarah Dean \\
Department of Computer Science\\
Cornell University\\
Ithaca, NY 14853, USA\\
\texttt{sdean@cornell.edu}
}

\iclrfinalcopy
\begin{document}

\maketitle

\begin{abstract}
Probabilistic downscaling models the conditional distribution of high-resolution fields given coarse inputs and is a core challenge in atmospheric science, climate modeling, and other multiscale physical systems. A widely used paradigm decomposes the problem into a deterministic mean predictor followed by a stochastic residual generator.
While effective in idealized settings, this mean--residual approach frequently produces biased and underdispersive ensembles in real-world applications. We identify a fundamental source of this failure: \textit{residual target misspecification}, where the residual distribution induced during training systematically differs from the correction distribution required at test time, a mismatch amplified by downscaling bias.
We theoretically connect this misspecification to underdispersion through the spread--skill ratio and show that conditional residual distribution matching controls both residual-energy and conditional-mean mismatch.
Motivated by this analysis, we introduce \textbf{ReMatch} (Residual Distribution Matching), which aligns training residual targets toward a held-out calibration regime via optimal transport.
On a controlled synthetic benchmark with varying bias levels and a real-world HRRR--ERA5 wind-field downscaling task, ReMatch substantially reduces underdispersion, improves calibration, and outperforms strong mean--residual and super-resolution baselines.

\end{abstract}

\section{Introduction}
Probabilistic downscaling seeks to model the conditional distribution of high-resolution fields given coarse-resolution inputs, enabling downstream tasks that require fine-scale dynamics such as hazard assessment, local forecasting, trajectory planning, and uncertainty-aware decision making.
% This problem arises broadly across scientific and engineering domains, including climate modeling, multiscale physical systems, speech synthesis, and medical imaging. 
% Fine-scale dynamics are essential for downstream tasks such as hazard assessment, local forecasting, trajectory planning, and uncertainty-aware decision making, but directly simulating them is often computationally expensive.
While downscaling inherently has unresolved subgrid variability, real-world applications face challenges like systematic biases and mismatches between coarse- and fine-resolution data sources—collectively referred to as \emph{downscaling bias}— which introduce additional, often domain- or source-dependent discrepancies~\citep{maraun2018statistical}. These biases distort the low-to-high resolution relationship far beyond simple interpolation, rendering the required high-resolution correction both stochastic and systematically more difficult to infer.

% A widely adopted strategy for such tasks is the two-stage mean--residual decomposition: a deterministic regressor first predicts a coarse approximation of the target, after which a stochastic generative model produces a residual correction that provides fine-scale detail and probabilistic spread. This design has been successfully applied across probabilistic downscaling~\citep{mardani2025residual,fotiadis2025adaptive,rampal2024enhancing}, climate forecasting~\citep{yu2024diffcast,nguyen2025raindiff}, speech synthesis~\citep{chen2022resgrad}, and medical imaging~\citep{zhangpixel}.
A two-stage mean–residual decomposition provides a useful framework for modeling both effects: a deterministic model predicts the conditional mean, after which a stochastic generative model learns the remaining residual distribution. Variants of this strategy have been adopted in various real-world tasks such as atmospheric downscaling~\citep{mardani2025residual}, weather and geophysical forecasting~\citep{couairon2026arches,finn2024generative}, precipitation nowcasting~\citep{yu2024diffcast,nguyen2025raindiff}, medical image super-resolution~\citep{han2024discrete}, probabilistic time-series and energy forecasting~\citep{lai2025rdit,chen2025windresidual}, and speech synthesis~\citep{chen2022resgrad}.
Despite their benefits, mean–residual generative models exhibit underdispersive predictions across applications and models, including atmospheric downscaling~\citep{mardani2025residual}, sea-ice forecasting~\citep{finn2024generative}, and probabilistic weather forecasting~\citep{couairon2026arches}.
% In atmospheric downscaling, CorrDiff~\citep{mardani2025residual} is a prominent representative: it trains a regression model on variables such as temperature, wind, and radar reflectivity, then employs a conditional diffusion model to generate stochastic residual corrections. However, CorrDiff and subsequent work consistently report that the resulting ensembles are under-dispersive, with ensemble spread significantly smaller than the ensemble mean error~\citep{mardani2025residual,fotiadis2025adaptive}.

We argue that this underdispersion stems from a fundamental issue that we term \emph{residual target misspecification}. In the mean–residual framework, the residual is an induced correction target defined relative to the fitted mean predictor. Consequently, the residual distribution observed during training can systematically differ from the corrections required at test time, causing the residual generator to underestimate the required predictive variability.

To address this, we propose \textbf{ReMatch} (Residual distribution Matching). Using a calibration set as a proxy for the test regime, ReMatch explicitly aligns the training residual distribution with the test-time residual distribution via optimal transport~\citep{wan2023debias}. This preserves the statistical benefits of the mean--residual decomposition while providing the residual generator with targets that better reflect the corrections required at inference.

Our main contributions are threefold: (i) we formalize residual target misspecification and establish its connection to ensemble underdispersion through the spread–skill ratio (SSR); (ii) we propose ReMatch, a practical residual distribution matching method that improves both predictive accuracy and ensemble calibration; and (iii) we provide a comprehensive empirical evaluation on a synthetic benchmark~\citep{chung2023turbulence} with systematically varied LR–HR bias and a real-world HRRR–ERA5 wind downscaling task~\citep{dowell2022high,hersbach2020era5}. Across these settings, ReMatch consistently outperforms strong deterministic and probabilistic baselines, supporting residual target misspecification as a major contributor to underdispersion in biased downscaling.

\section{Preliminaries}
\label{sec:background}
\subsection{Probabilistic Downscaling}

We consider probabilistic downscaling of high-resolution geophysical fields
\(
y \in \mathbb{R}^{c_{\mathrm{out}} \times H \times W}
\)
from low-resolution inputs
\(
x \in \mathbb{R}^{c_{\mathrm{in}} \times h \times w},
\)
where \(H>h\) and \(W>w\). Here, \(c_{\mathrm{in}}\) and \(c_{\mathrm{out}}\) denote the numbers of input and output channels, respectively, such as velocity, temperature or other physical variables, while \((h,w)\) and \((H,W)\) denote the low- and high-resolution spatial grid sizes. The task is to model the conditional distribution of \(y\) given \(x\). We use \(X\) and \(Y\) to denote the corresponding random input and target fields, with \(x\) and \(y\) denoting their realizations.

In many physical downscaling problems, this mapping is fundamentally non-deterministic.  A useful abstraction is to assume that the low-resolution observation is obtained from the high-resolution state through an effective degradation process
\(
x = D(y,s),
\)
where \(D\) may include not only resolution reduction, but also regridding, projection, source mismatch, simulator bias, and other systematic discrepancies, and \(s\) denotes latent source- or pipeline-dependent factors. For example, an ERA5-to-HRRR wind downscaling task combines differences in spatial resolution, numerical model formulation, data assimilation, pressure-level definition, and grid projection within the same effective degradation process \(D\). In idealized super-resolution benchmarks, \(D\) is often close to a simple downsampling operator, whereas in real-world settings it may reflect substantial bias and cross-source
mismatch. In either case, \(D\) is many-to-one, so the inverse problem is ill-posed. Hence the natural learning target is not a unique deterministic inverse, but the conditional distribution.
This probabilistic viewpoint is important not only because the inverse problem is many-to-one, but also because downstream applications often require uncertainty-aware predictions. In multiscale physical systems, unresolved fine-scale variability can affect forecasting, decision making, and motion planning of agents moving through the predicted field. Successful downscaling should therefore model both a central prediction and the spread of physically plausible
high-resolution realizations.

\subsection{Mean--Residual Decomposition} 
\label{subsec:mean_res_decomp}
We study a class of pipelines in which an upstream mean predictor $\mu_\phi(x) \in \mathbb{R}^{c_{\mathrm{out}} \times H \times W}$ first produces a deterministic approximation of the target, and a residual generator then models the remaining.
For a realization \((x,y)\), define the predictor-induced residual as $r_\phi(x,y) := y-\mu_\phi(x)$. A residual generator $q_\theta(r \mid x,\mu_\phi(x))$ then produces stochastic residual samples $\hat r \sim q_\theta(\,\cdot\mid x,\mu_\phi(x)),$ leading to $k$ number of ensemble predictions $\hat y^{(i)} = \mu_\phi(x) + \hat r^{(i)}, \quad i=1,\dots,k.$

% \yk{removed mentioning corrdiff here and kept it more general to mena-residual decomposition setting.}
% Such decomposition is widely adopted in many super-resolution tasks such as in  correction~\citep{mardani2025residual,lirescast,yu2024diffcast,nguyen2025raindiff,chen2022resgrad,li2025diffusion,lai2025rdit}. 
% CorrDiff~\citep{mardani2025residual} is a prominent\ns{popular state-of-the-art} example of this design in atmospheric downscaling: it first predicts a deterministic high-resolution field and then applies a diffusion model to generate a residual correction. 
% This decomposition separates the dominant predictable structure from the remaining stochastic variability, allowing the deterministic predictor to capture the primary signal while the residual generator focuses its capacity on unresolved fine-scale structure and uncertainty.
This separation allows the deterministic predictor to capture the primary predictable structure while the residual generator focuses its capacity on unresolved fine-scale variability and uncertainty. Moreover, the residual distribution can be easier to model than the full target distribution, as the upstream mean predictor removes much of the dominant, large-scale variation in the target, resulting in lower variance in the residuals~\citep{mardani2025residual}.
However, this advantage comes with an important caveat: the residual is defined relative to the upstream predictor. Once \(\mu_\phi\) is fixed, the residual stage is therefore trained on a predictor-induced target rather than an intrinsic component of the data distribution.
% For a random input--target pair \((X,Y)\), this target is
% \[
% R_\phi := Y - \mu_\phi(X).
% \]
This distinction is central to our analysis: the residual target can change when the reliability of \(\mu_\phi\) changes across domains, even if the underlying downscaling task remains the same.

% \subsection{Domains and Splits} \yk{let's use in and out instead or train, test since we will split in into train and calibration on rematch section.}
% For training and test splits, let $(X,Y)\sim P_d$, $d\in\{\mathrm{tr},\mathrm{te}\}$, denote the joint input--target distribution.
% Although the two splits are assumed to be i.i.d. samples from the same population, the fitted upstream predictor is not symmetric with respect to them: it is optimized on the training split and evaluated on held-out test samples. For a fixed \(\mu_\phi\), each split therefore induces its own empirical residual random variable
% \[
% R_d^\phi(X_d,Y_d) := Y_d - \mu_\phi(X_d),
% \qquad (X_d,Y_d)\sim P_d.
% \]
% As a result, its errors, and therefore the induced residuals \(R_{\mathrm{tr}}^\phi\) and \(R_{\mathrm{te}}^\phi\), can have systematically different distributions. This train--test residual mismatch is the object we analyze in the next section.

\subsection{In-Sample and Out-of-Sample Residuals}
\label{sec:domains}

% Let
% \(
% S_{\texttt{in}}=\{(x_i,y_i)\}_{i=1}^n
% \)
% denote the observations used to fit the deterministic predictor $\mu_{\phi}$. 
% We use \texttt{in} and \texttt{out} to distinguish whether an observation was used in the training mean--residual pipeline.
% \sd{make notation clear: lowercase are instances from the actual training data, upper cause is a rando variable (drawn uniformly from training data in d=in, otherwise from some new set)}
Let $\mathcal S=\{(x_i,y_i)\}_{i=1}^n$ denote the fitting sample used to obtain $\mu_{\phi}$. Conditional on the realized fitting sample $\mathcal S$, let $(X_{\mathtt{in}},Y_{\mathtt{in}})$ denote an observation drawn uniformly from $\mathcal S$, and let $(X_{\mathtt{out}},Y_{\mathtt{out}})\sim P $ denote an independent observation from the same population. Unless otherwise stated, we suppress conditioning on $\mathcal S$ for notational simplicity.
% Let $(X_{\mathrm{in}},Y_{\mathrm{in}})$ and $(X_{\mathrm{out}},Y_{\mathrm{out}})$ denote i.i.d. input-target pairs drawn from the same data distribution, where the former is used in fitting $\mu_\phi$ and the latter is not. 
For $d\in\{{\texttt{in},\texttt{out}\}}$, define the corresponding residual random variable as
\begin{equation}
R_d := Y_d-\mu_\phi(X_d).
\label{eq:residual}
\end{equation}
% \[
% R_d := Y_d-\mu_\phi(X_d).
% \]
Even when the underlying observations are sampled from the same population, the fitted predictor is statistically coupled to the observations used in fitting but not to fresh observations. Consequently, $R_{\texttt{in}}$ and $R_{\texttt{out}}$ need not have the same distribution. 
% Section~\ref{sec:misspecification} studies how this predictor-induced residual mismatch affects ensemble calibration.
We denote their conditional mean and variance by $m_d(x):=\mathbb E[R_d\mid X_d=x]$ and $V_d(x):=\operatorname{tr}\operatorname{Cov}(R_d\mid X_d=x)$ and define the residual energy as
\begin{equation}
M_d:=\mathbb E\|R_d\|_2^2 =\mathbb E\left[\|m_d(X_d)\|_2^2+V_d(X_d)\right].
\label{eq:residual_energy}
\end{equation}

\subsection{Spread-Skill Ratio}
The spread--skill ratio (SSR) measures whether the variability represented by a predictive ensemble is equal to its prediction error. For a probabilistic model that generates predictions $\hat Y \sim p(\hat Y\mid X)$, we consider the squared SSR,
% \sd{maybe more natural to mathematically define and then say spread over skill}
\begin{equation}
\label{eq:standard_ssr}
% \mathrm{SSR}
% =
% \frac{\mathrm{Spread}}{\mathrm{Skill}}, \qquad
\mathrm{SSR}^2=
\frac{
\mathbb E\left[\|\hat Y-\mathbb E[\hat Y\mid X]\|_2^2\right]
}{
\mathbb E\left[\|Y-\mathbb E[\hat Y\mid X]\|_2^2\right]
}.
\end{equation}
The numerator measures the ensemble variance, while the denominator measures the squared error of the ensemble mean. A calibrated predictive distribution has $\mathrm{SSR}=1$; values below 1 indicate underdispersion, whereas values above 1 indicate overdispersion.

\section{Residual Target Misspecification}
Mean--residual generative models have repeatedly exhibited underdispersive behavior. \citet{mardani2025residual} reports ensemble spreads that are systematically too small relative to ensemble-mean errors in atmospheric downscaling, while similar behavior has been observed in residual diffusion for sea-ice forecasting~\citep{finn2024generative}. More recently, \citet{couairon2026arches} links underdispersion in probabilistic weather forecasting to overfitting of the deterministic predictor. In this section, we formalize this phenomenon as \emph{residual target misspecification} and show how the mean--residual framework itself can induce underdispersion.
\label{sec:misspecification}
\paragraph {Fitting-induced residual-energy optimism.}
% \sd{should we use upper or lower case here?}
In the mean–residual structure introduced in Section~\ref{subsec:mean_res_decomp},
\(
\hat Y=\mu_{\phi}(X)+\hat R,\quad \hat R\sim q_\theta(\cdot\mid X,\mu_{\phi}(X)),
\)
the corresponding residual target is induced by the fitted mean predictor as in Eq.~\ref{eq:residual}. 

Even without distribution shift, empirical risk minimization (ERM) favors lower residual energy on the samples used for fitting. Suppose $\mu_{\phi}$ is obtained by minimizing empirical squared error over i.i.d.\ samples from $P$. Averaging over the randomness of the fitting sample, ERM yields the expected residual-energy inequality
% \sd{can we state the Min$<$Mout immediately following deifnition of Md in terms of Rd, without needing m(x) and V(x)? }
\begin{equation}
\label{eq:erm_optimism}
\mathbb{E}_\mathcal S[M_{\texttt{in}}]\leq \mathbb{E}_\mathcal S[M_{\texttt{out}}],
\end{equation}
where the \texttt{out} residual is evaluated on an independent observation from the same distribution. Intuitively, the fitted predictor is explicitly selected to minimize error on the observed sample, whereas an independent sample receives no such optimization advantage. Appendix~\ref{app:residual_theory} formalizes this argument using an independent ghost sample and shows that~\Eqref{eq:erm_optimism} follows directly from ERM optimality. This fitting-induced optimism is a standard consequence of ERM~\citet{vapnik1999overview,efron1986biased}.
% : the fitted predictor is selected to minimize empirical risk on the training sample, but receives no such optimization advantage on an independent sample

% To connect the residual-energy gap to ensemble calibration, we decompose the residual energy into its conditional mean and variance components. Define mean $m_d(x):=\mathbb E[R_d\mid X_d=x]$ and variance $V_d(x):=\operatorname{tr}\operatorname{Cov}(R_d\mid X_d=x)$ then,
% \begin{equation}
% \label{eq:residual_energy}
% M_d
% =
% \mathbb E\left[
% \|m_d(X_d)\|_2^2+V_d(X_d)
% \right].
% \end{equation}

% \sd{after related Min and Mout, we can definte the two components of Md based on mean and variance}
\paragraph{How residual optimism produces under-dispersion.}
% \sd{better to have precise relations without too many approx, see my suggestions in the appendix}\yk{made the assumption stronger with exact matching to training target, then it'll remove all approximations.}
\label{par:underdispersion}
As in \eqref{eq:residual_energy}, we define mean and variance of predicted residual as $\hat m(x):=\mathbb E_{q_\theta}[\hat R\mid X=x]$ and $\hat V(x):=\operatorname{tr}\operatorname{Cov}_{q_\theta}(\hat R\mid X=x)$. Denote the corresponding predicted residual energy as $\hat M:=\mathbb E\left[\|\hat m(X)\|_2^2+\hat V(X)\right]$.
Substituting the mean--residual structure into~\Eqref{eq:standard_ssr} gives
\begin{equation}
\label{eq:ssr_general}
\mathrm{SSR}^2
=
\frac{\mathbb E[\hat V(X_{\texttt{out}})]}
{\mathbb E\left[
V_{\texttt{out}}(X_{\texttt{out}})
+
\|m_{\texttt{out}}(X_{\texttt{out}})-\hat m(X_{\texttt{out}})\|_2^2
\right]}.
\end{equation}
Define predictive residual-energy gap as $\Delta\hat M:=M_{\texttt{out}}-\hat M.$
Then, as shown in Appendix~\ref{app:ssr_proof},
\begin{equation}
\label{eq:ssr_condition}
\mathrm{SSR}\leq1
\quad\Longleftrightarrow\quad
\mathbb E\left[
\hat m(X_{\texttt{out}})^\top(m_{\texttt{out}}(X_{\texttt{out}})-\hat m(X_{\texttt{out}}))
\right]
\leq
\frac{\Delta\hat M}{2}.
% \mathbb E\left[
% \hat m(X_{\texttt{out}})^\top m_{\texttt{out}}(X_{\texttt{out}})
% \right]
% \leq
% E\left[
% \|\hat m(X_{\texttt{out}})\|^2\right]+
% \frac{\Delta\hat M}{2}.
\end{equation}
Thus, underdispersion depends jointly on the residual energy supplied by the generator and the mismatch between its predicted conditional residual mean and the out-of-sample residual mean.

To isolate residual target misspecification from residual-generator approximation error, consider the optimistic setting in which the residual generator exactly matches its in-sample training target,
\begin{equation}
\label{eq:moment_matching}
\hat m(x)=m_{\texttt{in}}(x),
\qquad
\hat V(x)=V_{\texttt{in}}(x).
\end{equation}
Here, $m_{\texttt{in}}$ and $V_{\texttt{in}}$ characterize the residual target on the fitting distribution, whereas $\hat m$ and $\hat V$ are moments of the learned residual generator, which is subsequently applied to unseen inputs $X_{\texttt{out}}$.

For a fixed fitting sample with $\Delta\hat M\geq0$, Eq.~\ref{eq:ssr_condition} includes two special cases. If $m_{\texttt{in}}(x)=0$ for all $x$, meaning that the fitted mean predictor $\mu_\phi(x)$ recovers the in-sample conditional mean, then Eq.~\ref{eq:moment_matching} implies $\hat m(x)=0$ for all $x$. Evaluating this function at the random test input $X_{\texttt{out}}$ therefore gives $\hat m(X_{\texttt{out}})=0$, while the predictive variance remains $\hat V(X_{\texttt{out}})$.
% Since $X_{\texttt{in}}$ and $X_{\texttt{out}}$ have the same marginal distribution, this gives $\hat M=M_{\texttt{in}}$ and hence $\Delta\hat M=M_{\texttt{out}}-M_{\texttt{in}}$,
% which we know from  Equation~\ref{eq:erm_optimism} to be nonnegative in expectation.
% therefore implies $\mathbb E_{\mathcal S}[\Delta\hat M]\geq0$: even after removing residual-generator approximation error, ERM induces residual-energy optimism on average over fitting samples.
\begin{equation}
\label{eq:ssr_energy}
\mathrm{SSR}
=
\sqrt{\frac{\hat M}{M_{\texttt{out}}}}
\leq1.
\end{equation}
Alternatively, if $\hat m(x)=m_{\texttt{out}}(x)$, meaning that there is no mismatch in the predicted residual and the target residual mean, then
\begin{equation}
\label{eq:ssr_variance}
\mathrm{SSR}
=
\sqrt\frac{\mathbb E[\hat V(X_{\texttt{out}})]}
{\mathbb E[V_{\texttt{out}}(X_{\texttt{out}})]}
\leq1.
\end{equation}
More generally, in this optimistic setting, underdispersion holds whenever the cross second moment between the in-sample and out-of-sample residual means satisfies 
\[
\mathbb E[\hat m(X_{\texttt{out}})^\top m_{\texttt{out}}(X_{\texttt{out}})]
\leq
\mathbb E[\|\hat m(X_{\texttt{out}})\|_2^2]
+
\frac{M_{\texttt{out}}-\hat M}{2}.
\]
% \sd{I still think this notation is confusing ultimately - I have a hard time really understanding what it means to have $\hat m = m_{in}$. it might help to separate the in/out into the notation inside the expectation instead?}
The right hand side sums the second moment of the predict residual mean (always positive) with half the residual energy gap (positive in expectation).
% The full derivation and approximate versions of these special cases are provided in Appendix~\ref{app:ssr_proof}.
Overall, the analysis above identifies fitting-induced residual-energy optimism as one mechanism through which predictor-induced residual targets can lead to underdispersion.
More generally, though, 
% \eqref{eq:ssr_general} shows that ensemble calibration depends on both the spread and conditional mean of the 
we see that the difference between the 
predicted residual distribution relative to the target residual distribution is the source of this problem. This motivates our focus on aligning residual distributions more broadly.
% rather than correcting residual energy alone.

% Eq. (6) also shows why correcting residual energy alone is insufficient: calibration depends jointly on the predicted residual energy and conditional residual mean. In contrast, exact conditional residual distribution matching, q_\theta(\widehat R\mid X=x)=P(R_{\texttt{out}}\mid X=x), simultaneously gives \widehat m=m_{\texttt{out}} and \widehat V=V_{\texttt{out}}, and therefore recovers \mathrm{SSR}=1. This motivates targeting residual distributions more broadly rather than correcting residual energy alone.

\paragraph{Empirical evidence and implications for ReMatch.}
To empirically examine the source of this predictive mismatch, we measure the residual target energies $M_{\texttt{in}}$ and $M_{\texttt{out}}$ induced by the fitted mean predictor. On the ERA5--HRRR dataset, the samplewise target residual magnitudes yield $M_{\texttt{in}}=0.55$ and $M_{\texttt{out}}=4.29$, corresponding to a substantial realized target-energy gap of $\Delta M=3.74$.
% Consistent with the residual-energy optimism identified above, we observe a substantial residual-energy gap in the real-world downscaling task using ERA5--HRRR dataset.
% From the samplewise target residual magnitudes, we estimate $M_\texttt{in}=0.55$ and $M_\texttt{out}=4.29$ yielding $\Delta{M}=3.74$.
However, residual-energy mismatch captures only one aspect of residual target misspecification: the full residual distributions also exhibit a large normalized Wasserstein discrepancy of $0.94$.
This motivates ReMatch to align the residual targets at the distribution level rather than correcting residual energy alone.
Transport increases the training residual energy from $0.55$ to $2.19$, closer to the \texttt{out} sample residual target energy, 4.29, while reducing the normalized Wasserstein discrepancy from $0.94$ to $0.34$.
Full distributional statistics are provided in Appendix~\ref{app:residual_mismatch}.

\section{ReMatch: Residual Distribution Matching}
\label{sec:method}

We introduce Residual Distribution Matching (ReMatch), a probabilistic downscaling framework that combines deterministic mean prediction, residual target distribution alignment, and conditional diffusion-based residual generation.
% Its core idea is to preserve the useful mean prior provided by a deterministic high-resolution predictor while adapting the residual targets toward the correction distribution required outside the fitting sample.

\subsection{Residual distribution matching for residual target misspecification}
\label{subsec:distribution_matching}

Section~\ref{sec:misspecification} shows that residual-energy optimism provides one mechanism for underdispersion, but correcting residual energy alone is not sufficient for calibration.
In particular, \eqref{eq:ssr_condition} shows that calibration depends jointly on the predicted residual energy and the mismatch between the predicted and target conditional residual means.

This motivates a stronger target: matching the conditional residual distribution itself.
For a fixed input $x$, define $Q_x := q_\theta(\hat R\mid X=x)$ and $P_x := P(R_{\texttt{out}}\mid X=x).$ Define the integrated conditional Wasserstein distance as $\mathcal W_{\mathrm{cond}}^2
:=
\mathbb E_{X_{\texttt{out}}}
\left[
W_2^2(Q_{X_{\texttt{out}}},P_{X_{\texttt{out}}})
\right].$
We next show that the $\mathcal W_{\mathrm{cond}}^2$ controls both components governing the SSR condition in \eqref{eq:ssr_condition}: the conditional residual-mean mismatch and the residual-energy mismatch.

{(i) Conditional residual-mean mismatch.}
For a fixed $x$, let $\gamma_x\in\Pi(Q_x,P_x)$ denote a coupling, i.e., a joint distribution whose marginals are $Q_x$ and $P_x$.
Applying Jensen's and Cauchy--Schwarz inequalities to $\mathbb E_{\gamma_x}[\hat R-R_{\texttt{out}}]$ and minimizing over all couplings gives
\[
% \boxed{
\mathbb E_{X_{\texttt{out}}}
\|\hat m(X_{\texttt{out}})-m_{\texttt{out}}(X_{\texttt{out}})\|_2^2
\leq
\mathcal W_{\mathrm{cond}}^2.
% }
\]
Thus, $\mathcal W_{\mathrm{cond}}\to0$ implies convergence of the predicted conditional residual mean to the out-of-sample conditional residual mean in $L^2$.
The full derivation is provided in Appendix~\ref{par:distributionmatching1}.

{(ii) Residual-energy mismatch.}
Since the square root of residual energy corresponds to the $L^2$ norm of the residual, the reverse triangle inequality and minimization over the conditional couplings give
\[
% \boxed{
\left|
\sqrt{\hat M}
-
\sqrt{M_{\texttt{out}}}
\right|
\leq
\mathcal W_{\mathrm{cond}}.
% }
\]
Thus, $\mathcal W_{\mathrm{cond}}\to0$ implies $\Delta\hat M\to0$.
The full derivation is provided in Appendix~\ref{par:distributionmatching2}.

Together, these results show that conditional residual distribution matching jointly controls the two sources of mismatch appearing in the general SSR condition: the residual-energy gap and the conditional residual-mean mismatch.
In particular, exact conditional distribution matching, $Q_x=P_x,\quad\forall x$ implies $\hat m(x)=m_{\texttt{out}}(x)$, $\hat V(x)=V_{\texttt{out}}(x)$ and $\Delta\hat M=0$. Consequently, by Eq.~\eqref{eq:ssr_general}, $\mathrm{SSR}=1$.
ReMatch approximates this ideal target using held-out calibration residuals as an observable proxy for the out-of-sample correction distribution.
\subsection{Residual distribution matching with optimal transport}
\label{subsec:ot}
Since the out-of-sample residual distribution is unavailable during training, ReMatch uses residuals from a held-out calibration split as its observable proxy. With the regression model trained as in Section~\ref{subsec:mean_res_decomp}, ReMatch then transports the training residual targets toward this calibration regime.
The calibration split is excluded from mean-predictor fitting and therefore serves as an observable proxy for the corrections required outside the fitting sample. We use the domain notation $d\in\{\texttt{tr},\texttt{cal},\texttt{te}\}$ to denote the training, calibration, and test splits, respectively.

At a high level, ReMatch replaces each training residual with an optimal-transport-weighted combination of selected calibration residuals.
Because the desired alignment is conditional on the low-resolution input, the transport cost combines residual similarity with low-resolution conditioning consistency.
This discourages transport between samples with incompatible large-scale inputs and provides a practical approximation to the conditional residual alignment motivated above.

Direct optimal transport in the full residual space is computationally expensive and can be statistically unstable for high-dimensional spatial fields~\citep{choi2023generative,balaji2020robust}.
We therefore compute the transport cost in compact PCA representations~\citep{agustsson2017optimal,yuan2024optimal,liu2018latent}.
We fit separate PCA bases to the concatenated training and calibration residuals and to their corresponding low-resolution inputs.
Let $z_i^{\texttt{tr}}$ and $z_j^{\texttt{cal}}$ denote residual PCA representations, and let $\ell_i^{\texttt{tr}}$ and $\ell_j^{\texttt{cal}}$ denote LR-conditioning PCA representations.
We define
\[
c_{ij}
=
\alpha
\|z_i^{\texttt{tr}}-z_j^{\texttt{cal}}\|_2^2
+
\lambda_{\mathrm{cond}}
\|\ell_i^{\texttt{tr}}-\ell_j^{\texttt{cal}}\|_2^2.
\]
Here, the first term encourages residual distribution alignment, while the second restricts transport to samples with compatible low-resolution conditions.

Using $c_{ij}$, we solve the optimal transport problem and obtain a transport plan $\pi$ between the training and calibration residual samples.
For each training sample $i$, let $\mathcal T_i^{(k)}$ denote the top-$k$ calibration samples selected according to the transport weights.
We renormalize the corresponding weights as $\tilde{\pi}_{ij}
=
\frac{\pi_{ij}}
{\sum_{j'\in\mathcal T_i^{(k)}}\pi_{ij'}},
j\in\mathcal T_i^{(k)},$ and construct the transported residual representation and residuals: 
\[
\tilde z_i
=
\sum_{j\in\mathcal T_i^{(k)}}
\tilde{\pi}_{ij}
z_j^{\texttt{cal}}. 
\qquad \tilde r_i=U_K\tilde z_i+\bar r, \qquad \tilde\mu_i=y_i-\tilde r_i.
\]
The resulting pair $(\tilde\mu_i,\tilde r_i)$ preserves the target identity while replacing the original fitting residual target with one shifted toward the held-out correction regime.
The choice of $k$, PCA dimensionality, transport constraints, and additional implementation details are provided in Appendix~\ref{appendix:optimal_transport}.

Finally, we train a conditional diffusion model \(q_\theta(r \mid x, \mu_\phi(x))\) to generate stochastic residual corrections with the dataset union of the transported training samples and original calibration samples $\mathcal{D}_{\mathrm{res}}=
\{(x_i,\tilde{\mu}_i,\tilde{r}_i)\}_{i=1}^{n_{\texttt{tr}}}
\cup
\{(x_j,\mu_{\phi^\star}(x_j),r_j^{\texttt{cal}})\}_{j=1}^{n_{\texttt{cal}}}.$
Algorithm~\ref{sec:alg} provides an overview of the three-stage ReMatch training pipeline. 

\textbf{Implementation Note.} ReMatch is modular: the mean predictor and stochastic residual generator can be instantiated with different architectures; implementation details are in Appendix~\ref{app:experiment_details}.

\section{Experiments and Results}
\label{sec:experiments}

We evaluate ReMatch on a controlled synthetic benchmark and a real-world atmospheric downscaling task.
The synthetic experiment validates the mechanism of bias-induced residual target misspecification by systematically varying the LR--HR bias level. The real-world experiment tests ReMatch under real ERA5-to-HRRR wind-field distribution shifts. We compare downscaling fidelity and probabilistic calibration, and show downstream impact through a particle-advection experiment, where ensemble trajectories reveal how calibration affects predicted motion through a flow field.

\textbf{Baselines and Variants.}
\label{subsec:baselines} We use two versions of ReMatch: ReMatch\textsubscript{U} with a UNet mean predictor and ReMatch\textsubscript{S} with a SwinIR~\citep{liang2021swinir} mean predictor.
We compare ReMatch against deterministic baselines, including UNet~\citep{song2020score}, Conv-FNO~\citep{wen2022u}, and SwinIR~\citep{liang2021swinir}, and probabilistic CorrDiff-based baselines. The latter include the original CorrDiff~\citep{mardani2025residual}
% , a reduced-capacity mean-predictor variant (CorrDiff\textsubscript{m}), 
conditional diffusion model, classifier-free guidance on residual diffusion (CFG~\citep{ho2022classifier}), finetuning baseline adapted from \citep{couairon2026arches} and uncertainty-conditioned variants (\(\mathrm{UC}_{q}\) and \(\mathrm{UC}_{\mu}\)).

Except for SwinLR and Conv-FNO, all regression models are UNet-based~\citep{song2020score}, and all diffusion models use an EDM-style residual diffusion formulation~\citep{karras2022elucidating} with a UNet-based score network. We implement conditioning by channel-wise concatenation. 
Details of each baseline and implementation are provided in Appendix~\ref{appendix:baselines}.

\textbf{Evaluation Metrics.}
\label{subsec:metrics}
We report RMSE, SSR, and Continuous Ranked Probability Score (CRPS) as primary metrics. RMSE measures reconstruction accuracy, while SSR and CRPS evaluate probabilistic quality; SSR is particularly useful for diagnosing ensemble under-dispersion. We report LPIPS~\citep{zhang2018unreasonable} as an auxiliary perceptual metric, following prior use in atmospheric-field prediction~\citep{yu2024diffcast, hokson2026predicting}. We also report ACC~\citep{nai2025boosting}, computed between predicted and ground-truth (GT) deviations from the training GT mean across channels and times.
Metrics in Table~\ref{tab:main_results} are computed using 600 random test samples. We normalize LPIPS channelwise with quantile extrema, and average over ensembles for stochastic-model LPIPS.
%network architecture

\subsection{Bias-Controlled Synthetic Benchmark}
\label{subsec:toy}
To validate our theoretical analysis, we conduct a controlled experiment on the BLASTNet dataset~\citep{chung2023turbulence}. We inject structured downscaling bias of increasing severity into the low-resolution inputs while keeping the high-resolution targets fixed (details and bias generation procedure in Appendix~\ref{appendix:blastnet_details}). We use CorrDiff as a baseline method and ReMatch\textsubscript{U} as our proposed method. In all cases, 12 ensemble predictions are generated to calculate spreads and RMSE. 
\par
\begin{wrapfigure}{r}{0.5\linewidth}
    % \vspace{-1.3em}
    \centering
    \includegraphics[width=\linewidth]{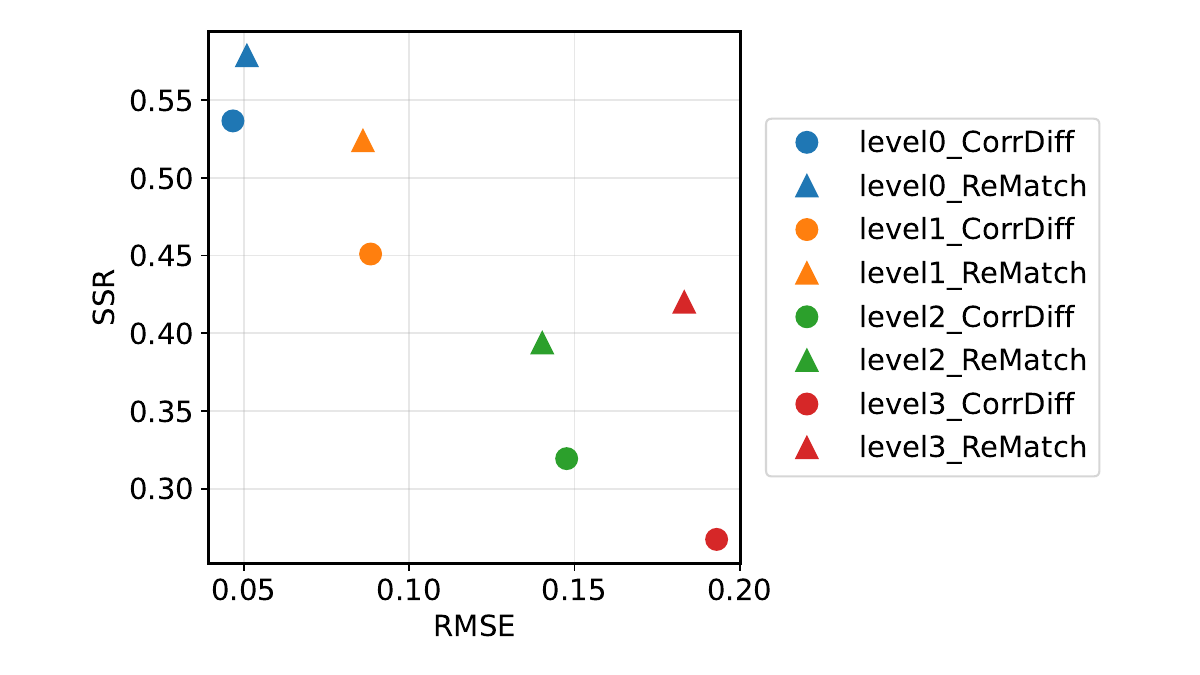}
    \caption{\textbf{Performance under increasing LR input bias.}
    ReMatch consistently improves SSR and reduces RMSE in biased settings, with the performance gap becoming more pronounced as the bias level increases.(SSR$\rightarrow1$, RMSE$\downarrow$)}
    \label{fig:blastnet_rmse_ssr}
    \vspace{-10mm}
\end{wrapfigure} 
\textbf{Dataset.}
We use the BLASTNet Momentum128 3D super-resolution dataset~\citep{chung2023turbulence} as a controlled testbed. The original dataset provides clean paired low- and high-resolution fields with 4 channels. We use the \(8\times\) super-resolution setting, with \(16 \times 16\) LR inputs and \(128 \times 128\) HR targets.
We define four bias levels (Level 0 to 3), where Level 0 uses the original clean interpolated low-resolution field and Levels 1--3 introduce progressively stronger structured mismatch by mixing 2D fields along the $z$ axis with Gaussian weights (full procedure in Appendix~\ref{appendix:blastnet_details}). Figure~\ref{fig:blastnet_level_comparison} illustrates how the input changes across the four bias levels under same HR target.

\subsubsection{Results}

\textbf{Does residual misspecification explain under-dispersion?}
Figure~\ref{fig:blastnet_rmse_ssr} shows the RMSE--SSR scatter across increasing LR--HR bias levels. As the bias level increases, the baseline exhibits a clear degradation: RMSE increases while SSR decreases. This trend indicates that stronger LR--HR mismatch induces a more difficult residual correction regime and leads to increasingly under-dispersive ensembles.

\textbf{Does ReMatch mitigate residual target misspecification?}
ReMatch improves RMSE in all biased settings except Level 0, where the LR input is not biased. It also consistently improves SSR, which becomes more pronounced at higher bias levels. These results support our central hypothesis: when paired low- and high-resolution data contain structured bias, the mean--residual decomposition can induce residual target misspecification, and residual distribution matching can mitigate this effect.
Channel-wise RMSE and SSR results are provided in Appendix~\ref{appendix:blastnet_details}.
% \begin{center}
% \begin{minipage}[t]{0.52\linewidth}
%     \centering
%     \includegraphics[width=\linewidth]{fig/blastnet_rmse_ssr.pdf}
%     \captionof{figure}{\textbf{Performance under increasing LR input bias.}
%     ReMatch consistently improves SSR and reduces RMSE in biased settings, with the performance gap becoming more pronounced as the bias level increases.(SSR$\rightarrow1$, RMSE$\downarrow$)}
%     \label{fig:blastnet_rmse_ssr}
% \end{minipage}
% \hfill
% \begin{minipage}[t]{0.43\linewidth}
%     \centering
%     \includegraphics[width=0.9\linewidth]{fig/data_index_49_1_level_comparison.pdf}
%     \captionof{figure}{\textbf{Qualitative comparison across LR input bias levels.} Rows correspond to LR bias levels. Columns show the GT HR field, LR input, ReMatch and CorrDiff prediction.}
%     \label{fig:blastnet_level_comparison}
% \end{minipage}
% \end{center}
% \subsection{Real-World Wind Field Downscaling Benchmark}
% \label{subsec:main_results}
% \begin{figure}
%     \centering
%     \includegraphics[width=0.5\linewidth]{fig/blastnet_rmse_ssr.pdf}
%     \caption{\textbf{Performance under increasing LR input bias.}
%     ReMatch consistently improves SSR and reduces RMSE in biased settings, with the performance gap becoming more pronounced as the bias level increases.(SSR$\rightarrow1$, RMSE$\downarrow$)\yk{figure position to be rearranged}}
%     \label{fig:blastnet_rmse_ssr}
% \end{figure}
\subsection{Real-world wind field downscaling benchmark}
Our main experiment evaluates ReMatch on a real-world atmospheric downscaling task with realistic distribution shift. 
Deterministic baselines (UNet, ConvFNO, and SwinIR) produce a single prediction, whereas stochastic methods (Diffusion, CorrDiff, CFG, UC\textsubscript{\(\mu\)}, UC\textsubscript{q}, Finetuning, ReMatch\textsubscript{U}, and ReMatch\textsubscript{S}) generate 12 ensemble members. Results for UC\textsubscript{q} are provided in Appendix~\ref{app:experiment_details}.

\textbf{Dataset.}
\label{dataset:windfield}
We use ERA5~\citep{hersbach2020era5} as the coarse-resolution input and HRRR~\citep{dowell2022high} as the high-resolution target over the northeast CONUS region. 
The input contains \(u\)- and \(v\)-wind fields at five stratospheric pressure levels, giving 10 channels in total and 4 temporal channels for sinusoidal embeddings of sample day and month on a \(21 \times 21\) grid. The target consists of the same 10 wind-component channels on a \(168 \times 168\) grid. The dataset contains hourly fields from 2018 to 2021. We use 2018--2020 for training and 2021 for testing; for methods requiring calibration data, we hold out 2020 as a calibration set. Further details are in Appendix ~\ref{appendix:hrrr_era5_details}. 
\subsubsection{Results}
\begin{figure}[t]
    \centering
    \includegraphics[width=1.\linewidth]{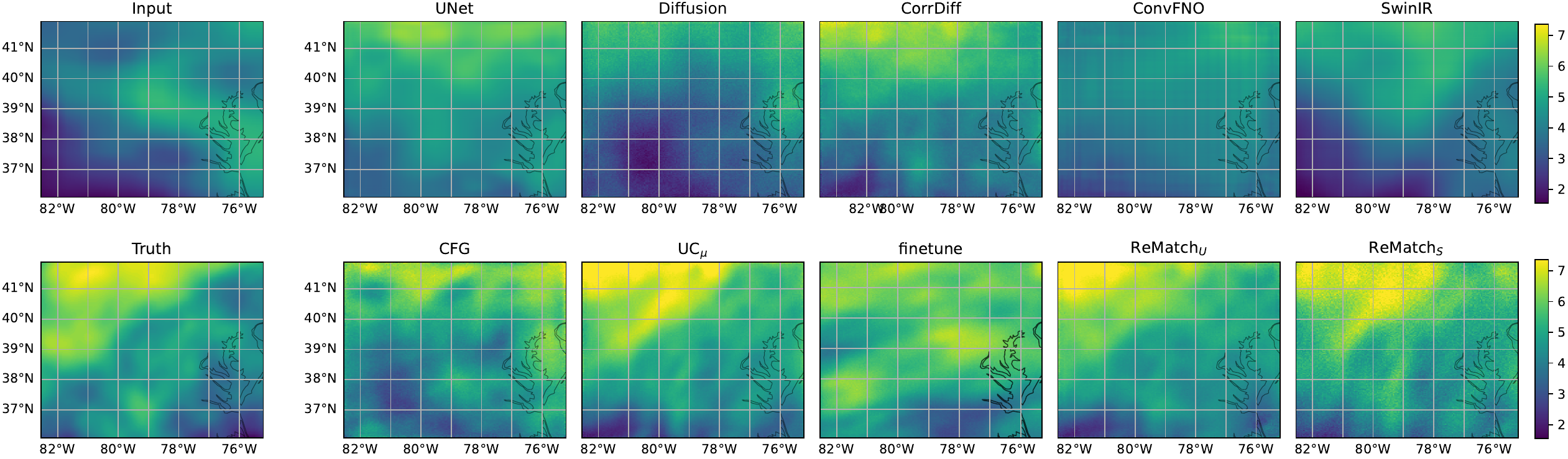}
    \caption{\textbf{Qualitative comparison of downscaled 50 hPa \(v\)-wind fields on 2021-06-05 10:00:00.} The top-left panel shows the upsampled ERA5 input, and the bottom-left panel shows the corresponding HRRR GT. Deterministic baselines produce a single prediction. Diffusion-based methods generate 12 ensemble members -- we only show the 7th ensemble member for stochastic methods.}
    \label{fig:compare_all}
% \vspace{-2.em}
\end{figure}
\textbf{What does realistic LR--HR mismatch look like?}
The first column of Fig.~\ref{fig:compare_all} illustrates LR--HR mismatch: even for the same atmospheric scene, the LR input and the corresponding HR target differ substantially. In the appendix, Fig.~\ref{fig:era_hrrr_gap} shows more examples where RMSE can quantify this mismatch, and Table~\ref{tab:era_hrrr_channelwise_stats} shows that the amount of bias realistically depends on channel. The cross-source mismatch makes real-world downscaling more challenging than idealized super-resolution. 
% \begin{table}[t]
% \centering
% \caption{\textbf{Quantitative comparison on HRRR-ERA5 pressure-level super-resolution over all channels.}}
% \label{tab:main_results}
% \resizebox{\linewidth}{!}{
% \begin{tabular}{l|ccc|cccccc|cc}
% \toprule
% Metric 
% & UNet & ConvFNO & SwinIR
% & Diffusion & CFG & UQ\textsubscript{$\mu$} & UQ\textsubscript{q} & Finetune
% & CorrDiff
% & \textbf{ReMatch\textsubscript{U}} & \textbf{ReMatch\textsubscript{S}} \\
% \midrule

% {RMSE} $\downarrow$
% & 1.8517 & 1.9213 & \underline{1.7642}
% & 1.8570 & 1.8309 & 1.8584 & 1.8551 & 2.015
% & 1.8342
% & 1.8102 & \textbf{1.7194} \\

% {SSR} $\rightarrow$ 1
% & -- & -- & --
% & 0.6419 & 0.3418 & \textbf{0.8745} & 0.5950 & 0.6499
% & 0.3241
% & 0.6347 & \underline{0.7150} \\

% {CRPS/MAE} $\downarrow$
% & 1.5420 & 1.6046 & 1.4538
% & 1.1908 & 1.2740 & 1.1545 & 1.1946 & \underline{1.1440}
% & 1.2859
% & 1.1567 & \textbf{1.0976} \\

% {LPIPS} $\downarrow$
% & 0.0737 & \underline{0.0715} & \textbf{0.0641}
% & 0.0943 & 0.0793 & 0.0871 & 0.0842 & 0.0931
% & 0.0904
% & 0.0841 & \underline{0.0715} \\

% {ACC} $\rightarrow$ 1
% & 0.8980 & 0.8957 & 0.9016
% & 0.8973 & 0.8994 & 0.8926 & 0.8952 & \underline{0.9046}
% & 0.8991
% & 0.9017 & \textbf{0.9051} \\

% \bottomrule
% \end{tabular}
% }
% \par
% \bigskip
% \footnotesize
% Notes: For deterministic models, CRPS is replaced with MAE and SSR is omitted. \\
% Best results are shown in bold, and second-best results are underlined.
% \vspace{-10.pt}
% \end{table}
% \vspace{-2.em}
\begin{table}[t]
\centering
\caption{\textbf{HRRR-ERA5 pressure-level super-resolution comparison over all channels.}}
\label{tab:main_results}
\resizebox{\linewidth}{!}{
\begin{tabular}{l|ccc|ccccc|cc}
\toprule
Metric 
& UNet & ConvFNO & SwinIR
& Diffusion & CFG & UQ\textsubscript{$\mu$} & Finetune
& CorrDiff
& \textbf{ReMatch\textsubscript{U}} & \textbf{ReMatch\textsubscript{S}} \\
\midrule

{RMSE} $\downarrow$
& 1.8517 & 1.9213 & \underline{1.7642}
& 1.8570 & 1.8309 & 1.8584 & 2.015
& 1.8342
& 1.8102 & \textbf{1.7194} \\

{SSR} $\rightarrow$ 1
& -- & -- & --
& 0.6419 & 0.3418 & \textbf{0.8745} & 0.6499
& 0.3241
& 0.6347 & \underline{0.7150} \\

{CRPS/MAE} $\downarrow$
& 1.5420 & 1.6046 & 1.4538
& 1.1908 & 1.2740 & 1.1545 & \underline{1.1440}
& 1.2859
& 1.1567 & \textbf{1.0976} \\

{LPIPS} $\downarrow$
& 0.0737 & \underline{0.0715} & \textbf{0.0641}
& 0.0943 & 0.0793 & 0.0871 & 0.0959
& 0.0904
& 0.0841 & \underline{0.0715} \\

{ACC} $\rightarrow$ 1
& 0.8980 & 0.8957 & 0.9016
& 0.8973 & 0.8994 & 0.8926 & \underline{0.9046}
& 0.8991
& 0.9017 & \textbf{0.9051} \\

\bottomrule
\end{tabular}
}
\par
\bigskip
\footnotesize
Notes: For deterministic models, CRPS is replaced with MAE and SSR is omitted. Best results are bolded, and second-best results are underlined.
%\vspace{-10.pt}
\vspace{-5.mm}
\end{table}
\textbf{Do the proposed methods improve downscaling performance?}
Figure~\ref{fig:compare_all} shows that UNet remains biased in overall scale. In contrast, mean--residual methods using UNet as the mean predictor, including CorrDiff, CFG, UC\textsubscript{$\mu$}, and ReMatch\textsubscript{U}, produce predictions that are closer to the HR target in both spatial structure and magnitude, suggesting that residual generation helps correct errors in the deterministic mean estimate. 
A similar trend is observed for ReMatch\textsubscript{S}, which uses SwinIR as the mean predictor. Overall, ReMatch\textsubscript{U} and ReMatch\textsubscript{S} produce visually competitive reconstructions while better recovering the large-scale magnitude and small-scale spatial patterns of the HR field. 

Table~\ref{tab:main_results} reports results on 600 test samples. ReMatch\textsubscript{S} achieves the best RMSE, CRPS, and ACC, consistent with the qualitative improvements in Figure~\ref{fig:compare_all}. Although UC\textsubscript{$\mu$} achieves the best SSR, it sacrifices performance on other metrics, particularly RMSE, while ReMatch\textsubscript{S} maintains a strong performance across all metrics. SwinIR achieves the best LPIPS but not on other metrics. ReMatch\textsubscript{U}, while not the best overall, consistently outperforms the baseline mean–residual model, CorrDiff, across all metrics, demonstrating the benefit of residual distribution matching.
Overall, these results show that ReMatch improves the accuracy–calibration tradeoff in real-world wind-field downscaling, with further gains when paired with a more expressive regression model.

\textbf{How much does the deterministic mean predictor matter?}
Among deterministic baselines, SwinIR achieves the strongest performance, showing that a stronger mean predictor can improve reconstruction accuracy. 
However, purely deterministic estimators, including UNet, ConvFNO, and SwinIR cannot quantify uncertainty. 
This limits their use in downstream tasks where calibrated ensemble spread is essential. ReMatch\textsubscript{S}, which adds residual generation on top of the SwinIR mean predictor, further improves both RMSE and CRPS/MAE, suggesting that ReMatch excels with stronger deterministic backbones while providing calibrated stochastic predictions.
% \textbf{Does ensemble calibration matter for downstream tasks?}
% Predicted physical fields are often used as environments for downstream decision making and agent motion planning. To examine how residual calibration may affect such downstream use, we advect particles through 48 consecutive hourly wind fields. Figure~\ref{fig:trajectory_comparison} compares trajectories from the HR ground-truth field, CorrDiff ensembles, and ReMatch\textsubscript{U} ensembles. CorrDiff trajectories deviate from the ground-truth path as the particle moves away from the initial position, whereas ReMatch\textsubscript{U} remains better centered around the truth while preserving ensemble spread. 

% While this example alone does not establish a general downstream advantage, it suggests that better-calibrated spread can improve the usefulness of ensemble forecasts for downstream uncertainty-aware tasks.
% \begin{figure}
%     \centering
%     \includegraphics[width=0.5\linewidth]{fig/trajectory_comparison_corrdiff_vs_rematch.pdf}
%     \caption{\textbf{Particle trajectory comparison using ensemble wind-field predictions over 48 consecutive hourly fields.} The red line denotes the trajectory advected by the HR ground-truth field, orange lines denote CorrDiff ensemble trajectories, and blue lines denote ReMatch\textsubscript{U} ensemble trajectories.}
%     \label{fig:trajectory_comparison}
% \end{figure}
\section{Discussion}
% \vspace{-1.5em}
\label{sec:discussion}
\textbf{Role of calibration data.}
ReMatch relies on a held-out calibration set because residual target adaptation is only meaningful when the calibration residuals reflect the correction regime expected at test time. This is a weaker and more practical assumption than requiring access to test labels, but it is still an important one: ReMatch is not intended to solve arbitrary distribution shift without representative held-out data. In our main experiments, we use a train--calibration--test split of approximately \(2{:}1{:}1\) to provide sufficient calibration samples to estimate the target residual distribution while retaining enough data for training the mean predictor and residual generator.
The appropriate split is task-dependent: a larger calibration set can better represent the test residual regime and stabilize residual transport, but it reduces model training data given a fixed total dataset size and increases the computational cost of optimal transport.
We empirically examine this trade-off by varying the train-to-calibration ratio on BLASTNet. We found that using more calibration data generally improves SSR while maintaining comparable RMSE, whereas an extremely small calibration set can substantially reduce SSR due to an insufficient representation of the target residual distribution.
Detailed discussion of the calibration-set sensitivity are provided in Appendix~\ref{app:calibration_sensitivity}.

\textbf{Does ensemble calibration matter for downstream tasks?} Predicted physical fields are often used as environments for downstream decision making and agent motion planning. To examine how residual calibration may affect such downstream use, we advect particles through 48 consecutive hourly wind fields. 
\par
\begin{wrapfigure}{r}{0.52\linewidth}
    \vspace{-1.3em}
    \centering
    \includegraphics[width=\linewidth]{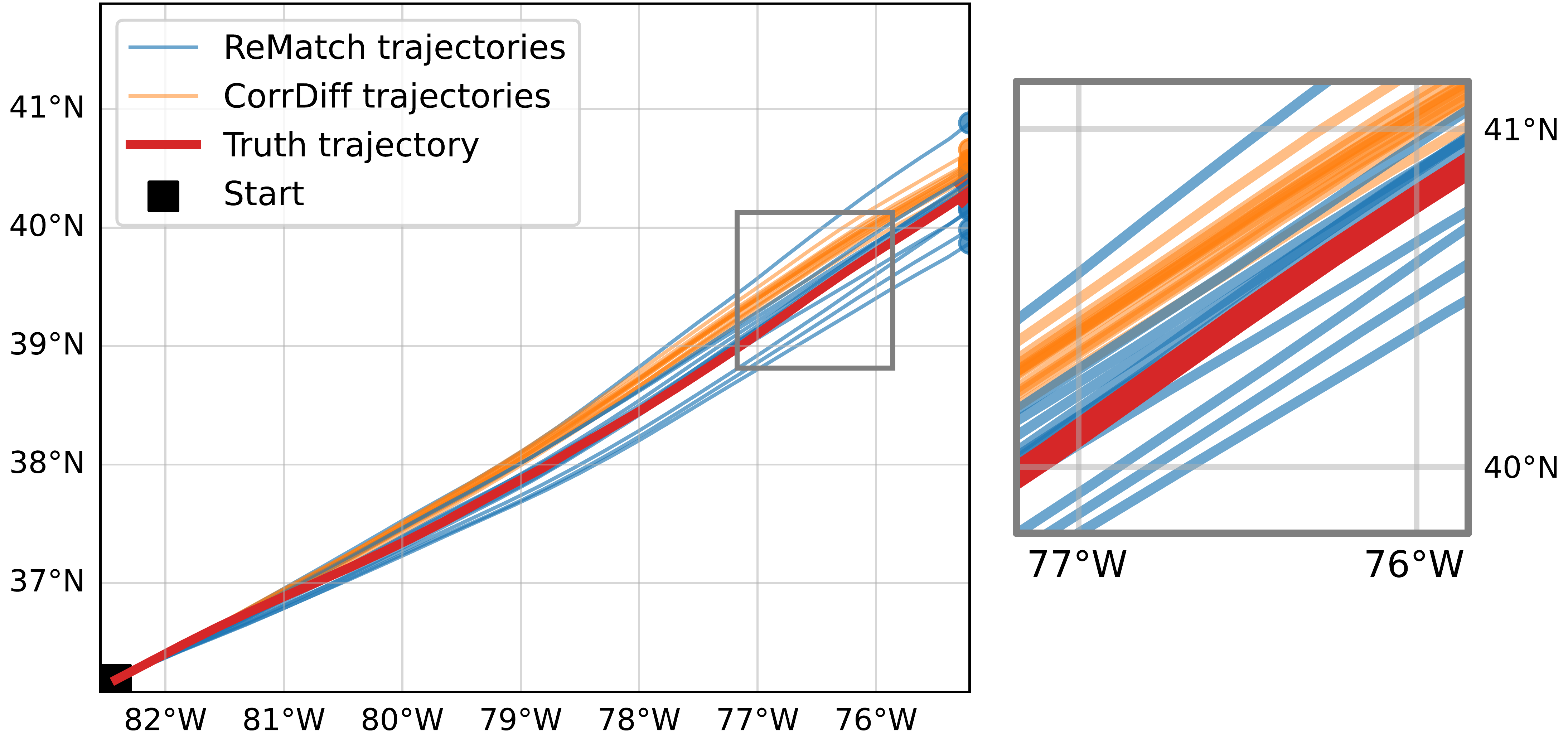}
    \caption{\textbf{Particle trajectory comparison using ensemble wind-field predictions over 48 consecutive hourly fields.} The red line denotes the trajectory advected by the HR ground-truth field, orange lines denote CorrDiff ensemble trajectories, and blue lines denote ReMatch\textsubscript{U} ensemble trajectories.}
    \label{fig:trajectory_comparison}
    \vspace{-1.2em}
\end{wrapfigure} 
Figure~\ref{fig:trajectory_comparison} compares trajectories from the HR ground-truth field, CorrDiff ensembles, and ReMatch\textsubscript{U} ensembles. CorrDiff trajectories deviate from the ground-truth path as the particle moves away from the initial position, whereas ReMatch\textsubscript{U} remains better centered around the truth while preserving ensemble spread. 
While this example alone does not establish a general downstream advantage, it suggests that better-calibrated spread can improve the usefulness of ensemble forecasts for downstream uncertainty-aware tasks.

\textbf{Limitations and future directions.}
ReMatch performs residual matching in a low-dimensional PCA space, which improves tractability but may not fully preserve nonlinear or localized residual structures. Learned or physics-aware representations could provide richer transport spaces. Our current implementation also performs transport independently across output channels, which does not explicitly preserve cross-channel dependencies; joint or structure-aware transport is an important direction for preserving coupled physical structure. 
% Scalability is another limitation, as transport is computed over the full training and calibration sets. Large-scale OT solvers~\citep{genevay2016stochastic,cuturi2013sinkhorn} or learned transport maps~\citep{korotin2022neural} could alleviate this cost. 
Finally, extending ReMatch beyond the mean--residual frameworks studied here, such as to deterministic-guided diffusion, flow matching, or quantile regression, is an interesting direction for future work.

\textbf{Conclusion.}
\label{sec:conclusion}
% We studied probabilistic downscaling through residual target misspecification, where a fitted mean predictor can induce a mismatch between training and test-time residual targets.
% We showed that this mismatch can manifest through both residual-energy and conditional residual-mean errors, and connected their correction through residual distribution matching to ensemble calibration.
% We introduced ReMatch, which transports training residuals toward a calibration regime while preserving the deterministic mean predictor.
% Across controlled and real-world downscaling experiments, ReMatch improves the accuracy–calibration tradeoff over strong baselines.
% Our results highlight that calibration failures can originate from the construction of residual targets, motivating distribution alignment at the target level rather than only improving the stochastic generator.
We studied probabilistic downscaling under real-world bias through the lens of residual target misspecification, where a fitted mean predictor induces residual targets whose training distribution can differ from the correction regime required at test time. We showed how this mismatch can lead to underdispersion through residual-energy and conditional residual-mean mismatch, and established that residual distribution matching jointly controls these two sources of miscalibration. Motivated by this connection, we introduced ReMatch, which adapts training residuals toward a held-out calibration regime using optimal transport. Across controlled bias experiments, real-world ERA5-to-HRRR wind downscaling, and a downstream particle-advection example, ReMatch improves the accuracy–calibration tradeoff over strong baselines.

More broadly, our results suggest that calibration failures in residual generative pipelines should not be viewed only as deficiencies of the stochastic generator; they can originate from the predictor-induced targets on which it is trained. This perspective motivates addressing calibration at the target-distribution level, rather than only improving the generative model itself.
% \newpage
\section*{Acknowledgements}
We thank Preston Culbertson for helpful discussions on diffusion underdispersion in residual generative downscaling. This work was supported in part by the NVIDIA Academic Grant Program. This work was partly funded by NSF CCF 2312774, NSF OAC-2311521, NSF IIS-2442137, NSF IIS2505098, a gift to the LinkedIn-Cornell Bowers CIS Strategic Partnership, and an AI2050 Early Career Fellowship program at Schmidt Science.

\bibliography{iclr2027_conference}
\bibliographystyle{iclr2027_conference}
\newpage
\appendix
\section{Appendix}
\subsection{Related Work}

\paragraph{Statistical and deep learning downscaling.}
Statistical downscaling models high-resolution fields from coarse inputs using learned relationships between resolutions~\citep{maraun2018statistical}. Classical approaches such as bias correction and quantile mapping are effective for marginal statistics, but often struggle to preserve spatial coherence and multivariate dependencies. Deep neural networks have improved deterministic downscaling accuracy through convolutional, transformer-based, and operator-learning architectures, including SwinIR~\citep{liang2021swinir} and U-FNO~\citep{wen2022u}. However, deterministic predictors produce only a single estimate and therefore cannot represent the spread of physically plausible high-resolution states. This has motivated probabilistic downscaling methods based on modern generative models~\citep{rampal2024enhancing}, including flow matching~\citep{fotiadis2025adaptive} and diffusion-based approaches~\citep{wan2023debias, springenberg2026diffscale}, which aim to generate calibrated ensembles for uncertainty-aware downstream applications.

\paragraph{Diffusion-based mean--residual downscaling.}
Diffusion models have recently become a strong framework for probabilistic atmospheric downscaling because they can model complex high-dimensional conditional distributions. CorrDiff~\citep{mardani2025residual} is a representative mean--residual approach: a deterministic regression model first predicts a high-resolution mean field, and a conditional diffusion model then generates stochastic residual corrections. Similar residual or two-stage formulations appear in precipitation nowcasting and forecasting models such as DiffCast~\citep{yu2024diffcast}, RainDiff~\citep{nguyen2025raindiff}, ResCast~\citep{lirescast}, and CasCast~\citep{gong2024cascast}. While this decomposition reduces the variance that the diffusion model must learn, prior work reports persistent under-dispersion in generated ensembles~\citep{mardani2025residual, fotiadis2025adaptive}. Our work studies this failure mode directly and argues that it arises from predictor-induced residual target misspecification, rather than only from limitations of the stochastic sampler.

\paragraph{Cross-source/model mismatch vs. predictor-induced target mismatch.} Recent work has also considered robustness across heterogeneous climate and weather data sources. R2-D2~\citep{lopez2025dynamical} combines dynamical and generative downscaling, using an intermediate WRF simulation to map outputs from different Earth-system models to a common effective representation before learning a residual diffusion model. \citet{molinaro2026universal} instead train a model-agnostic conditional diffusion downscaler that generalizes zero-shot across heterogeneous upstream forecasting systems. These approaches demonstrate that generative downscaling can generalize across source models, but address source mismatch through common physical preprocessing or source-agnostic conditional generation. ReMatch targets a complementary failure mode: the mismatch between the residual distribution induced on samples used to fit an upstream mean predictor and the correction distribution required on held-out samples. Rather than removing source dependence from the input representation, ReMatch explicitly aligns these predictor-induced residual targets using held-out calibration data.

\paragraph{Distribution matching and domain adaptation.}
Our approach is related to domain adaptation, where source and target distributions are aligned to improve transfer under distribution shift~\citep{ben2010theory, pan2010survey}. Classical covariate-shift and importance-weighting methods reweight source samples to match the target input distribution under the assumption that the conditional target remains unchanged~\citep{shimodaira2000improving,sugiyama2012machine}, while OT-based domain adaptation aligns source and target feature-label distributions through a learned transport plan~\citep{courty2017ot,courty2017joint,redko2019multisource}. In downscaling, prior work uses OT to debias the coarse input before conditional diffusion sampling~\citep{wan2023debias}. ReMatch differs in that it does not primarily debias the input distribution. Instead, it aligns the predictor-induced residual target distribution that the stochastic generator is trained to model.

\paragraph{Distribution matching with projected transport.}
Optimal transport provides a principled framework for comparing and aligning probability distributions~\citep{villani2009optimal,peyre2019computational}. Entropic regularization improves the computational scalability of empirical OT~\citep{cuturi2013sinkhorn}. However, estimating Wasserstein distances in high-dimensional spaces is statistically challenging, with convergence rates that deteriorate with dimension~\citep{fournier2015rate,weed2019sharp}. A common remedy is to compare distributions through lower-dimensional projections or compact representations, as in sliced Wasserstein distances~\citep{bonneel2015sliced}, max-sliced and subspace-robust variants~\citep{deshpande2019max, paty2019subspace, lin2020projection}, and latent-space transport or regularization methods~\citep{tolstikhin2018wae, kolouri2018swae}. ReMatch follows this projection-based perspective by computing transport costs in a latent representation of residuals and conditioning inputs, rather than directly in the full residual field. This provides a tractable way to match residual correction targets, while suggesting that richer physics-aware representations or scalable projected distribution-matching objectives could further improve residual alignment in future work.
\newpage

\subsection{Additional Theory for Residual Target Misspecification}
\label{app:residual_theory}

\subsubsection{Residual Optimism under Empirical Risk Minimization}

Let $S=\{Z_i\}_{i=1}^n$, where $Z_i=(X_i,Y_i)\overset{\mathrm{iid}}{\sim}P$, and define the squared loss, empirical risk, and population risk as
\[
\ell(\phi;Z):=\|Y-\mu_\phi(X)\|_2^2,\qquad
L_S(\phi):=\frac{1}{n}\sum_{i=1}^n\ell(\phi;Z_i),\qquad
L_P(\phi):=\mathbb E_{Z\sim P}[\ell(\phi;Z)].
\]
Suppose the mean predictor is obtained by exact empirical risk minimization,
\[
\hat\phi_S\in\arg\min_{\phi\in\Phi}L_S(\phi).
\]
Under the residual-energy definition in~\eqref{eq:residual_energy}, the in-sample and out-of-sample residual energies for a realized fitting sample $\mathcal S$ correspond to
\[
M_{\texttt{in}}=L_S(\hat\phi_S),\qquad
M_{\texttt{out}}=L_P(\hat\phi_S).
\]
where, as in the main text, dependence on $\mathcal S$ is suppressed. Averaging over the randomness of the fitting sample, we obtain the following result.
\paragraph{Proposition.}
If $\hat\phi_S$ is an empirical risk minimizer trained on $S\sim P^n$, then
\begin{equation}
\label{eq:erm_optimism_appendix}
\mathbb{E}_\mathcal S[M_{\texttt{in}}]\leq \mathbb{E}_\mathcal S[M_{\texttt{out}}].
\end{equation}

\paragraph{Proof.}
Let $S'\sim P^n$ be an independent copy of $S$, and let $\hat\phi_{S'}$ be the corresponding empirical risk minimizer. Since $\hat\phi_S$ minimizes empirical risk on $S$,
\[
L_S(\hat\phi_S)\leq L_S(\hat\phi_{S'}).
\]
Taking expectation over $S$ and $S'$ yields
\[
\mathbb E_{S,S'}[L_S(\hat\phi_S)]
\leq
\mathbb E_{S,S'}[L_S(\hat\phi_{S'})].
\]

Since $L_S(\hat\phi_S)$ does not depend on $S'$, the left-hand side becomes
\[
\mathbb E_{S,S'}[L_S(\hat\phi_S)]
=
\mathbb E_S[L_S(\hat\phi_S)]
=
\mathbb{E}_\mathcal S[M_{\texttt{in}}].
\]

For the right-hand side, conditional on $S'$, the predictor $\hat\phi_{S'}$ is independent of $S$. Since the observations in $S$ are i.i.d.\ draws from $P$,
\[
\mathbb E_S[L_S(\hat\phi_{S'})\mid S']
=
L_P(\hat\phi_{S'}).
\]
Taking expectation over $S'$ then gives
\[
\mathbb E_{S,S'}[L_S(\hat\phi_{S'})]
=
\mathbb E_{S'}[L_P(\hat\phi_{S'})].
\]
Because $S$ and $S'$ are identically distributed,
\[
\mathbb E_{S'}[L_P(\hat\phi_{S'})]
=
\mathbb E_S[L_P(\hat\phi_S)]
=
\mathbb{E}_\mathcal S[M_{\texttt{out}}].
\]
Combining the two sides gives
\[
\mathbb E_{\mathcal S}[M_{\texttt{in}}]
\leq
\mathbb E_{\mathcal S}[M_{\texttt{out}}].
\]
\hfill$\square$
\hfill$\square$

The result requires no distribution shift: both the fitting sample and the independent observation defining $L_P$ are drawn from the same distribution $P$. It requires neither linearity nor Gaussianity. The inequality is non-strict because equality is possible when fitting does not benefit from adaptation to the observed sample. The result assumes exact ERM; predictors obtained by approximate or finite-step optimization need not satisfy the inequality exactly.
\subsubsection{A Linear-Gaussian Illustration}

To illustrate a setting in which the inequality is strict, consider
\[
Y=X\beta^\star+\epsilon,\qquad
\epsilon\sim\mathcal N(0,\sigma^2I_n),
\]
where $X\in\mathbb R^{n\times p}$ has full column rank. Ordinary least squares gives
\[
\hat Y=HY,\qquad
H=X(X^\top X)^{-1}X^\top.
\]
The training residual is
\[
R_{\texttt{in}}=(I-H)\epsilon.
\]
Conditional on $X$,
\[
\mathbb E\left[\|R_{\texttt{in}}\|_2^2\mid X\right]
=
\sigma^2\operatorname{tr}(I-H)
=
\sigma^2(n-p),
\]
where $\operatorname{tr}(H)=p$. Therefore, the expected residual energy per observation is
\begin{equation}
\label{eq:ols_min}
\mathbb{E}_\mathcal S[M_{\texttt{in}}]
=
\sigma^2\left(1-\frac{p}{n}\right).
\end{equation}

For a fresh observation $y_0=x_0^\top\beta^\star+\epsilon_0$ independent of the training sample,
\[
y_0-x_0^\top\hat\beta
=
\epsilon_0-x_0^\top(X^\top X)^{-1}X^\top\epsilon,
\]
and hence
\[
\mathbb E\left[
(y_0-x_0^\top\hat\beta)^2
\mid X,x_0
\right]
=
\sigma^2\left[
1+x_0^\top(X^\top X)^{-1}x_0
\right].
\]
If the rows of $X$ and $x_0$ are i.i.d.\ $\mathcal N(0,\Sigma)$, then
\[
X^\top X\sim\mathcal W_p(\Sigma,n),
\]
and, for $n>p+1$,
\[
\mathbb E[(X^\top X)^{-1}]
=
\frac{\Sigma^{-1}}{n-p-1}.
\]
Since $\mathbb E[x_0x_0^\top]=\Sigma$,
\[
\mathbb E\left[
x_0^\top(X^\top X)^{-1}x_0
\right]
=
\operatorname{tr}\left(
\mathbb E[(X^\top X)^{-1}]\Sigma
\right)
=
\frac{p}{n-p-1}.
\]
Therefore,
\begin{equation}
\label{eq:ols_mout}
\mathbb E_{\mathcal S}[M_{\texttt{out}}]
=
\sigma^2\left(
1+\frac{p}{n-p-1}
\right).
\end{equation}
For $p>0$ and $n>p+1$,
\[
\mathbb E_{\mathcal S}[M_{\texttt{in}}]
=
\sigma^2\left(1-\frac{p}{n}\right)
<
\sigma^2
<
\sigma^2\left(1+\frac{p}{n-p-1}\right)
=
\mathbb E_{\mathcal S}[M_{\texttt{out}}].
\]
Thus, fitting alone can produce a strict residual-energy gap even when fitting and unseen observations follow exactly the same data-generating distribution. This example is illustrative only;~\eqref{eq:erm_optimism_appendix} itself requires neither linearity nor Gaussianity.

\subsubsection{Proof of the Spread--Skill Relation}
\label{app:ssr_proof}

Consider an out-of-sample input--target pair $(X,Y)$ and let
\[
\hat Y
=
\mu_{\phi}(X)+\hat R,
\qquad
\hat R
\sim
q_\theta(\cdot\mid X,\mu_{\phi}(X)),
\]
with the conditional mean and variance of the predicted residual defined as
\[
\hat m(x)=\mathbb E_{q_\theta}[\hat R\mid X=x],
\qquad
\hat V(x)=\operatorname{tr}\operatorname{Cov}_{q_\theta}(\hat R\mid X=x).
\]

Because $\mu_{\phi}(X)$ is deterministic conditional on $X$,
\[
\operatorname{tr}\operatorname{Cov}_{q_\theta}(\hat Y\mid X=x)
=
\operatorname{tr}\operatorname{Cov}_{q_\theta}(\hat R\mid X=x)
=
\hat V(x).
\]
Therefore,
\[
\mathrm{Spread}^2
=
\mathbb E[\hat V(X_{\texttt{out}})].
\]

The conditional ensemble mean is
\[
\mathbb E[\hat Y\mid X]
=
\mu_{\phi}(X)+\hat m(X).
\]
For an out-of-sample target,
\[
Y_{\texttt{out}}
=
\mu_{\phi}(X_{\texttt{out}})+R_{\texttt{out}},
\]
and hence
\[
Y_{\texttt{out}}-\mathbb E[\hat Y\mid X_{\texttt{out}}]
=
R_{\texttt{out}}-\hat m(X_{\texttt{out}}).
\]
Applying the conditional bias--variance decomposition gives
\[
\mathbb E\left[
\|Y_{\texttt{out}}-\mathbb E[\hat Y\mid X_{\texttt{out}}]\|_2^2
\mid X_{\texttt{out}}=x
\right]
=
V_{\texttt{out}}(x)
+
\|m_{\texttt{out}}(x)-\hat m(x)\|_2^2.
\]
Taking expectation over $X_{\texttt{out}}$ gives
\[
\mathrm{Skill}^2
=
\mathbb E\left[
V_{\texttt{out}}(X_{\texttt{out}})
+
\|m_{\texttt{out}}(X_{\texttt{out}})-\hat m(X_{\texttt{out}})\|_2^2
\right].
\]
Therefore,
\[
\mathrm{SSR}^2
=
\frac{
\mathbb E[\hat V(X_{\texttt{out}})]
}{
\mathbb E\left[
V_{\texttt{out}}(X_{\texttt{out}})
+
\|m_{\texttt{out}}(X_{\texttt{out}})-\hat m(X_{\texttt{out}})\|_2^2
\right]
}.
\]

Since $\mathrm{SSR}\leq1$ if and only if $\mathrm{Skill}^2-\mathrm{Spread}^2\geq0$, we have
\begin{align}
\mathrm{Skill}^2-\mathrm{Spread}^2
&=
\mathbb E\left[
V_{\texttt{out}}-\hat V
+
\|m_{\texttt{out}}-\hat m\|_2^2
\right] \nonumber\\
&=
\mathbb E\left[
V_{\texttt{out}}-\hat V
+
\|m_{\texttt{out}}\|_2^2-\|\hat m\|_2^2
\right]
+
2\mathbb E\left[
\|\hat m\|_2^2-\hat m^\top m_{\texttt{out}}
\right],
\label{eq:ssr_general_gap}
\end{align}
where all quantities inside the expectations are evaluated at $X_{\texttt{out}}$.

Define the predicted residual energy as
\[
\hat M
:=
\mathbb E\left[
\hat V(X_{\texttt{out}})
+
\|\hat m(X_{\texttt{out}})\|_2^2
\right].
\]
Since
\[
M_{\texttt{out}}
=
\mathbb E\left[
V_{\texttt{out}}(X_{\texttt{out}})
+
\|m_{\texttt{out}}(X_{\texttt{out}})\|_2^2
\right],
\]
we define the \emph{predictive residual-energy gap} as
\[
\Delta\hat M
:=
M_{\texttt{out}}-\hat M.
\]

Using this notation, Eq.~\eqref{eq:ssr_general_gap} becomes
\[
\mathrm{Skill}^2-\mathrm{Spread}^2
=
\Delta\hat M
+
2\mathbb E\left[
\hat m^\top(\hat m-m_{\texttt{out}})
\right].
\]
Thus, the exact condition for underdispersion is
\[
\boxed{
\mathrm{SSR}\leq1
\quad\Longleftrightarrow\quad
\mathbb E\left[
\hat m^\top(m_{\texttt{out}}-\hat m)
\right]
\leq
\frac{\Delta\hat M}{2}.
}
\]

\paragraph{Connection to ERM residual optimism.}
Equation~\eqref{eq:erm_optimism} establishes an optimism gap in the residual targets themselves,
\[
\mathbb E_{\mathcal S}[M_{\texttt{in}}]
\leq
\mathbb E_{\mathcal S}[M_{\texttt{out}}].
\]
The predictive residual energy $\hat M$, in contrast, is defined by the learned residual generator evaluated on out-of-sample inputs and therefore need not equal $M_{\texttt{in}}$.
The ERM result thus identifies a mechanism by which the residual target available during fitting can systematically underestimate the correction energy required out of sample.
For a learned residual generator that inherits this target-level optimism, $\Delta\hat M>0$, the exact SSR condition above characterizes when this predictive energy gap leads to underdispersion.

Thus, after removing residual-predictor approximation error, the predictive residual-energy gap reduces to the fitting-induced residual-energy optimism established by ERM.
For a fixed fitting sample with $\Delta\hat M\geq0$, the exact underdispersion condition becomes
\[
\boxed{
\mathbb E\left[
\hat m^\top(m_{\texttt{out}}-\hat m)
\right]
\leq
\frac{\Delta\hat M}{2}.
}
\]
This condition includes several useful regimes in which the fitting-induced residual-energy gap leads to underdispersion.
First, if $\hat m(x)=0$, the mean-dependent term vanishes and $\Delta\hat M\geq0$ is sufficient for $\mathrm{SSR}\leq1$.
Second, if $\hat m(x)=m_{\texttt{out}}(x)$, the mean-dependent term also vanishes and $\Delta\hat M\geq0$ is again sufficient for $\mathrm{SSR}\leq1$.
More generally, neither equality is necessary, and underdispersion holds whenever
\[
\mathbb E\left[\hat m^\top m_{\texttt{out}}\right]
\leq
\mathbb E\left[\|\hat m\|_2^2\right]
+
\frac{\Delta\hat M}{2}.
\]
Thus, the in-sample and out-of-sample residual means may be positively aligned and need not match exactly, as long as their cross second moment does not exceed the in-sample residual-mean second moment by more than half of the residual-energy gap.

\paragraph{Special case 1: negligible in-sample residual mean.}
Consider first the case in which the predicted residual mean is negligible, $\hat m(x)\approx0$.
The residual-energy decomposition gives
\[
\hat M
=
\mathbb E\left[
\hat V(X_{\texttt{out}})
+
\|\hat m(X_{\texttt{out}})\|_2^2
\right]
\approx
\mathbb E[\hat V(X_{\texttt{out}})],
\]
so the numerator of the SSR expression is approximately $\hat M$.
Under the same assumption,
\[
\begin{aligned}
&\mathbb E\left[
V_{\texttt{out}}(X_{\texttt{out}})
+
\|m_{\texttt{out}}(X_{\texttt{out}})
-\hat m(X_{\texttt{out}})\|_2^2
\right]\\
&\qquad\approx
\mathbb E\left[
V_{\texttt{out}}(X_{\texttt{out}})
+
\|m_{\texttt{out}}(X_{\texttt{out}})\|_2^2
\right]
=
M_{\texttt{out}}.
\end{aligned}
\]
Therefore,
\[
\mathrm{SSR}
\approx
\sqrt{\frac{\hat M}{M_{\texttt{out}}}},
\]
which recovers Eq.~\eqref{eq:ssr_energy} in Section~\ref{par:underdispersion}.
This case shows that when the fitted mean predictor leaves negligible systematic residual bias on its fitting samples, the spread--skill relation is governed directly by the train--test residual-energy gap.

\paragraph{Special case 2: matched conditional residual means.}
Consider instead the case in which the conditional residual mean generalizes across the fitting and out-of-sample regimes, $\hat m(x)\approx m_{\texttt{out}}(x)$.
The mean-mismatch contribution to the skill then approximately vanishes, giving
\[
\mathrm{SSR}^2
\approx
\frac{
\mathbb E[\hat V(X_{\texttt{out}})]
}{
\mathbb E[V_{\texttt{out}}(X_{\texttt{out}})]
}.
\]
Thus, when the conditional residual means are matched across the two regimes, underdispersion is governed directly by the conditional residual-variance gap.
In the exact case $\hat m(x)=m_{\texttt{out}}(x)$, the mean-energy terms also cancel from the predictive residual-energy gap,
\[
\Delta\hat M
=
M_{\texttt{out}}-\hat M
=
\mathbb E\left[
V_{\texttt{out}}(X_{\texttt{out}})
-
\hat V (X_{\texttt{out}})
\right].
\]
Hence, for a fixed fitting sample,
\[
\mathrm{SSR}\leq1
\quad\Longleftrightarrow\quad
\mathbb E[V_{\texttt{out}}(X_{\texttt{out}})]
\geq
\mathbb E[\hat V(X_{\texttt{out}})].
\]
Averaged over fitting samples, the ERM residual-energy optimism therefore reduces to a residual-variance optimism in this special case.
\subsection{Empirical Characterization of Residual Target Misspecification}
\label{app:residual_mismatch}
We complement the theoretical analysis in Section~\ref{sec:misspecification} by directly characterizing the mismatch between the residual targets induced on the fitting and held-out samples. This analysis concerns the target residuals $R_{\texttt{in}}$ and $R_{\texttt{out}}$

For each sample, we compute the magnitude of the target residual $r_\phi(x,y)=y-\mu_\phi(x)$ and report its empirical mean and standard deviation across samples.
Because the residual energy is defined as $M_d=\mathbb E[\|R_d\|_2^2]$, the corresponding empirical second moment can be recovered from these statistics as
\begin{equation}
 M_d
=
{\mathbb E}[\|R_d\|_2^2]
=
\operatorname{Mean}(\|R_d\|_2)^2
+
\operatorname{Std}(\|R_d\|_2)^2.
\end{equation}
As shown in Table~\ref{tab:dist_shift}, the target residual magnitude increases from $0.73\pm0.13$ on the \texttt{in} samples to $2.01\pm0.50$ on the \texttt{out} samples. These statistics correspond to
\[
M_{\texttt{in}}
=
0.73^2+0.13^2
=
0.55,
\qquad
M_{\texttt{out}}
=
2.01^2+0.50^2
=
4.29,
\]
yielding a realized residual-energy gap of
\begin{equation}
\Delta M
=
M_{\texttt{out}}-M_{\texttt{in}}
=
3.74.
\end{equation}
The direction of this realized gap is consistent with the fitting-induced residual-energy optimism derived in Section~\ref{sec:misspecification}, which establishes the corresponding inequality in expectation over fitting samples.
The mismatch is not limited to residual energy.
The full residual target distributions exhibit a Wasserstein distance of $1.28$, corresponding to a normalized Wasserstein discrepancy of $0.94$, together with large Jensen--Shannon and Kolmogorov--Smirnov distances.
In comparison, the upstream mean predictions exhibit substantially smaller distributional differences, including a normalized Wasserstein discrepancy of only $0.06$.
These observations indicate that the residual targets exhibit distributional differences beyond the second-moment mismatch characterized by our theoretical analysis.

ReMatch addresses this broader mismatch by transporting the training residual targets toward residuals obtained from the held-out calibration regime.
After transport, the residual magnitude becomes $1.43\pm0.38$, corresponding to an empirical residual energy of
\[
 \tilde M_{\texttt{in}}
=
1.43^2+0.38^2
=
2.19.
\]
Thus, transport moves the residual energy from $0.55$ toward the held-out value of $4.29$ while simultaneously reducing the normalized Wasserstein discrepancy from $0.94$ to $0.34$.
The Wasserstein, Jensen--Shannon, and Kolmogorov--Smirnov distances all decrease after transport, showing that ReMatch does not merely rescale residual energy but moves the training residual targets toward the held-out residual distribution more broadly.

\begin{table}[t]
\centering
\caption{In--out distribution mismatch of the target residuals and upstream mean predictions on the ERA5--HRRR dataset~\ref{dataset:windfield}. \textbf{Parentheses ( )} report the corresponding values after transporting the training residuals toward the held-out calibration residual regime. Distributional differences are measured using Wasserstein ($W_2$), normalized Wasserstein, Jensen--Shannon (JS), and Kolmogorov--Smirnov (KS) distances. Out-of-sample statistics are unchanged by transport.}
\label{tab:dist_shift}
\resizebox{\linewidth}{!}{
\begin{tabular}{l|cc|cc|cccc}
\toprule
Variable & Mean$_{\texttt{in}}$ & Mean$_{\texttt{out}}$ & Std$_{\texttt{in}}$ & Std$_{\texttt{out}}$ & $W_2$ & Norm.\ $W_2$ & JS & KS \\
\midrule
$r_\phi(x,y)$ & 0.73 (1.43) & 2.01 & 0.13 (0.38) & 0.50 & 1.28 (0.58) & 0.94 (0.34) & 0.98 (0.61) & 0.99 (0.59) \\
$\mu_\phi(x)$ & 14.98 (15.03) & 14.22 & 7.74 (7.72) & 7.19 & 0.89 (0.90) & 0.06 (0.06) & 0.23 (0.23) & 0.09 (0.08) \\
\bottomrule
\end{tabular}
}
\end{table}
\subsection{Distribution matching for residual target misspecification}
\label{subsec:distribution_matching}
\paragraph{Wasserstein distance controls the conditional residual-mean mismatch.}
\label{par:distributionmatching1}
For any coupling $\gamma_x\in\Pi(Q_x,P_x)$,
\[
\begin{aligned}
\|\hat m(x)-m_{\texttt{out}}(x)\|_2
&=
\left\|
\mathbb E_{\gamma_x}
[\hat R-R_{\texttt{out}}]
\right\|_2\\
&\leq
\mathbb E_{\gamma_x}
\|\hat R-R_{\texttt{out}}\|_2
&&\text{(Jensen's inequality)}\\
&\leq
\left(
\mathbb E_{\gamma_x}
\|\hat R-R_{\texttt{out}}\|_2^2
\right)^{1/2}
&&\text{(Cauchy--Schwarz inequality)}.
\end{aligned}
\]
Taking the infimum over all couplings, by the definition of the $2$-Wasserstein distance,
\[
W_2^2(Q_x,P_x)
:=
\inf_{\gamma_x\in\Pi(Q_x,P_x)}
\mathbb E_{\gamma_x}
\|\hat R-R_{\texttt{out}}\|_2^2,
\]
gives
\[
\|\hat m(x)-m_{\texttt{out}}(x)\|_2
\leq
W_2(Q_x,P_x).
\]
Averaging the squared inequality over $X_{\texttt{out}}$ and using the definition
\[
\mathcal W_{\mathrm{cond}}^2
:=
\mathbb E_{X_{\texttt{out}}}
\left[
W_2^2(Q_{X_{\texttt{out}}},P_{X_{\texttt{out}}})
\right]
\]
therefore gives
\[
\boxed{
\mathbb E_{X_{\texttt{out}}}
\|\hat m(X_{\texttt{out}})
-
m_{\texttt{out}}(X_{\texttt{out}})\|_2^2
\leq
\mathcal W_{\mathrm{cond}}^2.
}
\]
Hence, $\mathcal W_{\mathrm{cond}}\to0$ implies convergence of the predicted conditional residual mean to the out-of-sample conditional residual mean in $L^2$.\paragraph{Wasserstein distance controls the residual-energy mismatch.}
\label{par:distributionmatching2}
Consider a joint coupling obtained by coupling the conditional residual distributions $Q_x$ and $P_x$ for each $x$.
The reverse triangle inequality in $L^2$ gives
\[
\begin{aligned}
\left|
\sqrt{\hat M}
-
\sqrt{M_{\texttt{out}}}
\right|
&=
\left|
\sqrt{\mathbb E\|\hat R\|_2^2}
-
\sqrt{\mathbb E\|R_{\texttt{out}}\|_2^2}
\right|\\
&\leq
\left(
\mathbb E\|\hat R-R_{\texttt{out}}\|_2^2
\right)^{1/2}.
\end{aligned}
\]
Taking the optimal conditional couplings yields
\[
\boxed{
\left|
\sqrt{\hat M}
-
\sqrt{M_{\texttt{out}}}
\right|
\leq
\mathcal W_{\mathrm{cond}}.
}
\]
Thus, $\mathcal W_{\mathrm{cond}}\to0$ also implies $\Delta\hat M\to0$.

Exact conditional residual distribution matching,
\[
Q_x=P_x
\qquad
\forall x,
\]
simultaneously gives
\[
\hat m(x)=m_{\texttt{out}}(x),
\qquad
\hat V(x)=V_{\texttt{out}}(x),
\qquad
\Delta\hat M=0.
\]
Substituting these relations into Eq.~\eqref{eq:ssr_general} gives
\[
\mathrm{SSR}^2
=
\frac{\mathbb E[V_{\texttt{out}}(X_{\texttt{out}})]}
{\mathbb E[V_{\texttt{out}}(X_{\texttt{out}})]}
=1.
\]

Thus, conditional residual distribution matching provides a stronger calibration target than correcting residual energy alone: it jointly controls the residual-energy gap and conditional residual-mean mismatch appearing in the general SSR condition.
ReMatch approximates this target using held-out calibration residuals as an observable proxy for the out-of-sample correction distribution.

\subsection{Algorithm Overview}\label{sec:alg}
\begin{algorithm}[h]
\caption{ReMatch training pipeline}
\label{alg:ReMatch}
\begin{algorithmic}[1]
\State \textbf{Input:} 
$\mathcal{D}_{\texttt{tr}} = \{(x_i, y_i)\}_{i=1}^{n_{\texttt{tr}}}$, 
$\mathcal{D}_{\texttt{cal}} = \{(x_j, y_j)\}_{j=1}^{n_{\texttt{cal}}}$
\State \textbf{Stage 1. Mean training:}
\State \hspace{1em} $\phi^\star \leftarrow \arg\min_\phi \mathcal{L}_{\mathrm{reg}}(\phi;\mathcal{D}_{\texttt{tr}})$
\State \textbf{Stage 2. Residual target alignment:}
\State \hspace{1em} Compute residuals $r_i^{\texttt{tr}} = y_i - \mu_{\phi^\star}(x_i)$ and $r_j^{\texttt{cal}} = y_j - \mu_{\phi^\star}(x_j)$
\State \hspace{1em} Project onto PCA: $z_i^{\texttt{tr}} = U_K^\top r_i^{\texttt{tr}}$, \quad $z_j^{\texttt{cal}} = U_K^\top r_j^{\texttt{cal}}$, \quad $\ell_i^{\texttt{tr}} = V^\top x_i$, $\ell_j^{\texttt{cal}} = V^\top x_j$
\State \hspace{1em} Define cost $c_{ij} = \alpha \|z_i^{\texttt{tr}}-z_j^{\texttt{cal}}\|_2^2 + \lambda_{\mathrm{cond}} \|\ell_i^{\texttt{tr}}-\ell_j^{\texttt{cal}}\|_2^2$
\State \hspace{1em} Compute sparse transport plan $\pi$
\State \hspace{1em} \text{Transport: select }$\mathcal T_i$, $\tilde z_i^{\rm tr}=\sum_{j\in\mathcal T_i}\tilde\pi_{ij}z_j^{\rm cal}$, $\tilde{r}_i = U_K \tilde{z}_i^{\texttt{tr}}$, \quad $\tilde{\mu}_i = y_i - \tilde{r}_i$
% \State \hspace{1em} Transport: $\tilde{z}_i^{\texttt{tr}} = \sum_j \pi_{ij} z_j^{\texttt{cal}} / \sum_j \pi_{ij}$, \quad $\tilde{r}_i = U_K \tilde{z}_i^{\texttt{tr}}$, \quad $\tilde{\mu}_i = y_i - \tilde{r}_i$
\State \hspace{1em} Define residual-learning dataset
$\mathcal{D}_{\mathrm{res}}=
\{(x_i,\tilde{\mu}_i,\tilde{r}_i)\}_{i=1}^{n_{\texttt{tr}}}
\cup
\{(x_j,\mu_{\phi^\star}(x_j),r_j^{\texttt{cal}})\}_{j=1}^{n_{\texttt{cal}}}.$

\State \textbf{Stage 3. Residual learning:}
\State \hspace{1em} $\theta^\star \leftarrow
\arg\min_\theta \mathcal{L}_{\mathrm{res}}(\theta;\mathcal{D}_{\mathrm{res}})$
\State \textbf{Return:} $\mu_{\phi^\star},\; q_{\theta^\star}$
\end{algorithmic}
\end{algorithm}
\subsection{Residual Distribution Matching via Optimal Transport}
\label{appendix:optimal_transport}
\label{appendix:pca}
Directly computing optimal transport over full residual fields is computationally expensive and can be overly sensitive to grid-scale noise. We therefore compute the transport cost in compact PCA representations of the residual targets and low-resolution conditioning inputs~\cite{agustsson2017optimal,yuan2024optimal,liu2018latent}. The PCA bases are fitted using the concatenated training and calibration samples, and each field is projected onto the leading components before transport.

For PCA truncation, we choose the number of modes so that the reconstruction fidelity exceeds \(99\%\) for all channels. For the BlastNet dataset~\cite{chung2023turbulence}, we use 200 modes for the high-resolution residual field and 70 modes for the low-resolution conditioning field. For the HRRR dataset~\cite{dowell2022high}, we use 100 residual modes and for ERA5~\cite{hersbach2020era5}, we use 30 modes. BlastNet requires more modes because it contains a wider variety of spatial patterns induced by varying PDE coefficients and flow conditions.

For a residual sample \(r_\phi\), let
\[
z = U_K^\top (r_\phi - \bar r) \in \mathbb{R}^{K}
\]
denote its residual PCA representation, where \(U_K\) is the residual PCA basis and \(\bar r\) is the residual mean used for centering. Similarly, for a low-resolution input \(x\), let
\[
\ell = V_{K_{\mathrm{cond}}}^\top (x - \bar x)
\in \mathbb{R}^{K_{\mathrm{cond}}}
\]
denote its conditioning representation, where \(V_{K_{\mathrm{cond}}}\) is the PCA basis for the low-resolution input. Let
\[
\{z_i^{\texttt{tr}}, \ell_i^{\texttt{tr}}\}_{i=1}^{n_{\texttt{tr}}}
\quad \text{and} \quad
\{z_j^{\texttt{cal}}, \ell_j^{\texttt{cal}}\}_{j=1}^{n_{\texttt{cal}}}
\]
denote the resulting low-dimensional representations for the training and calibration sets, respectively.

We compute transport separately for each output channel. For each channel, we define the pairwise cost between training sample \(i\) and calibration sample \(j\) as
\[
c_{ij}
=
\alpha
\left\|z_i^{\texttt{tr}} - z_j^{\texttt{cal}}\right\|_2^2
+
\lambda_{\mathrm{cond}}
\left\|\ell_i^{\texttt{tr}} - \ell_j^{\texttt{cal}}\right\|_2^2 .
\]
The first term encourages residual similarity, while the second term preserves consistency with the low-resolution conditioning input. This prevents the transported residual from being matched to a calibration residual whose associated low-resolution field is incompatible with the source input. In all experiments, we set \(\alpha=1.0\) and \(\lambda_{\mathrm{cond}}=1.0\).

To make the transport computation scalable, we construct a sparse \(k\)-nearest-neighbor graph in the augmented PCA feature space and solve a balanced entropic Sinkhorn problem only over the resulting candidate pairs~\cite{cuturi2013sinkhorn}. Specifically, we use \(k=32\), entropic regularization \(0.1\), and 1500 Sinkhorn iterations. The sparse transport plan \(\pi\) is computed with uniform source and target marginals. When the number of training and calibration samples differs, the target marginal simply assigns a larger mass to each calibration sample.

Because a sparse \(k\)-nearest-neighbor graph can leave some calibration samples with too few incoming source candidates, we augment the graph before Sinkhorn normalization. In particular, we enforce a minimum target in-degree proportional to the source-to-target sample ratio by adding nearest reverse-neighbor edges for under-covered calibration samples. This improves feasibility of the balanced transport constraints while keeping the sparse graph size fixed.

After computing the sparse plan, we construct transported residual coefficients using a weighted top-\(m\) selection rather than averaging over all \(k\) neighbors. For each training sample, we select the top \(m=3\) calibration candidates according to the transport weights, with a capacity constraint limiting how often each calibration sample can be reused. This reduces collapse onto a small subset of calibration residuals while avoiding overly diffuse barycentric averages. The selected weights are renormalized within the top-\(m\) set.

For each training sample \(i\), the transported residual representation is therefore computed as
\[
\tilde{z}_i^{\texttt{tr}}
=
\sum_{j \in \mathcal{T}_i}
\tilde{\pi}_{ij} z_j^{\texttt{cal}},
\]
where \(\mathcal{T}_i\) denotes the selected top-\(m\) calibration candidates and \(\tilde{\pi}_{ij}\) denotes the corresponding normalized transport weight. We then reconstruct the transported residual in pixel space as
\[
\tilde{r}_i
=
U_K \tilde{z}_i^{\texttt{tr}} + \bar r .
\]

Finally, we define the corresponding pseudo-mean by enforcing exact consistency with the original high-resolution target:
\[
\tilde{\mu}_i := y_i - \tilde{r}_i .
\]
This yields adapted pseudo training samples
\[
(x_i, \tilde{\mu}_i, \tilde{r}_i)
\]
that preserve the mean--residual decomposition
\[
y_i = \tilde{\mu}_i + \tilde{r}_i
\]
by construction. Thus, ReMatch shifts the residual target distribution toward the calibration residual regime while maintaining sample-wise consistency among the input, pseudo-mean, residual, and ground-truth target.
\newpage
\subsection{Synthetic Bias-Controlled Benchmark Details}
\label{appendix:blastnet_details}
\begin{figure}[h]
\centering
\includegraphics[width=0.9\linewidth]{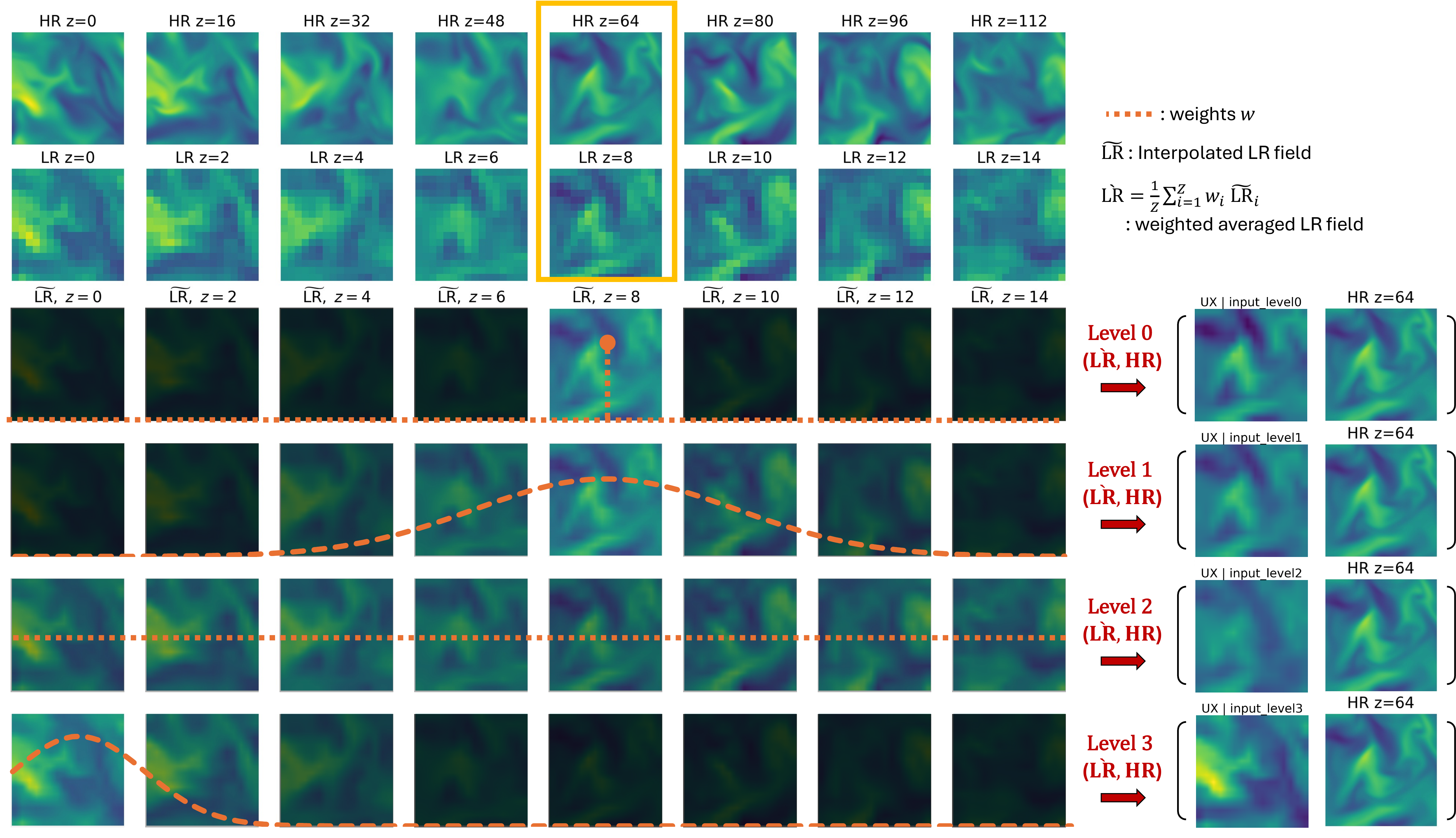}
\caption{Synthetic LR bias levels used in the BlastNet benchmark. Level 0 uses the original interpolated LR input. Levels 1--3 introduce progressively stronger structured mismatch by mixing slices along the \(z\)-axis with Gaussian weights.}
\label{fig:blastnet_bias_levels}
\end{figure}
We use the BLASTNet Momentum128 3D super-resolution dataset~\cite{chung2023turbulence} as a controlled synthetic benchmark. The dataset provides clean paired low- and high-resolution turbulent flow fields, where the low-resolution inputs are generated from the corresponding high-resolution DNS fields through filtering and downsampling. Each high-resolution sample is a three-dimensional field of size \(128 \times 128 \times 128\), with corresponding low-resolution fields provided at multiple downsampling scales. Both the input and target fields contain four physical channels: density \texttt{RHO} and the three velocity components \texttt{UX}, \texttt{UY}, and \texttt{UZ}. We use clusters 4, 11, 14, and 17, and take the \(8\times\) downsampled inputs (spatial resolution \(16 \times 16\)). Since our implementation operates on 2D inputs, we extract slices along the \(z\)-axis. This also enables controlled bias injection along the depth direction.

We define four bias levels. Level 0 uses the original interpolated low-resolution field (clean interpolation). For levels 1--3, we replace each low-resolution slice with a weighted mixture of slices along the \(z\)-axis using Gaussian weights
\[
w(z)=\exp\!\left(-\frac12\left(\frac{z-c}{\sigma}\right)^2\right),
\]
where the center \(c\) and bandwidth \(\sigma\) are chosen to produce progressively stronger structured mismatch: local mixture (\(\sigma=2\)) for level 1, broad averaging for level 2, and sharp distant centering (\(\sigma=1\)) for level 3.

Figure~\ref{fig:blastnet_bias_levels} illustrates the four bias levels and the corresponding weighting patterns. The baseline method~\cite{mardani2025residual} is trained using the entire training set of 6064 samples. For ReMatch, we partition this training set into a training subset (\(n=4048\)) and a calibration set (\(n=2016\)). The test set contains 2032 samples.

\subsubsection{Controlled Synthetic Study of Bias-induced Residual Misspecification}
Figure~\ref{fig:blastnet_channel_wise} shows channel-wise RMSE and SSR across four bias levels. Across all channels, ReMatch yields increasingly larger SSR improvements as the bias level increases, with only minor RMSE degradation or, in some cases, RMSE improvement. Figure~\ref{fig:blastnet_level_comparison} provides a qualitative comparison between the baseline method and ReMatch across the four input bias levels. Qualitatively, ReMatch better preserves high-frequency features than CorrDiff, particularly as the bias level increases.
\begin{figure}[h]
\centering
\includegraphics[width=1\linewidth]{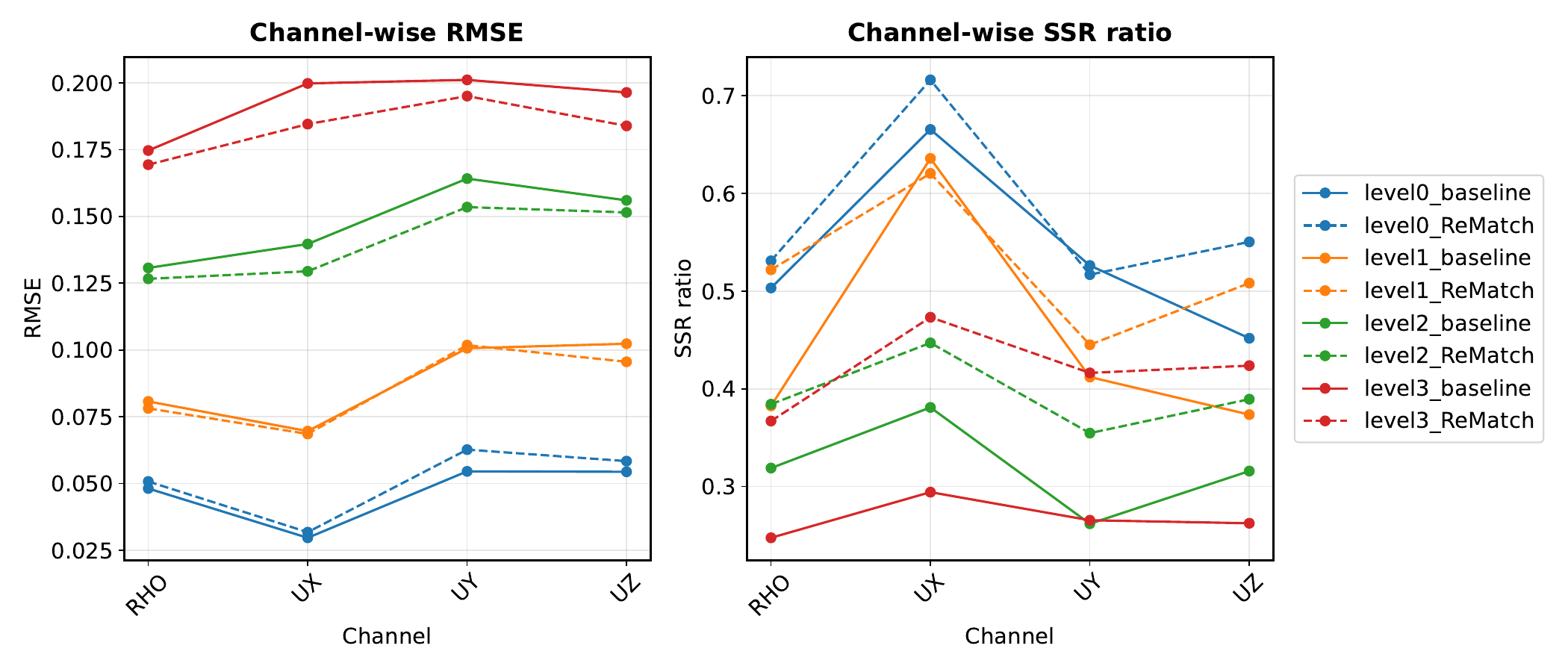}
\caption{Channel-wise RMSE and SSR ratio over different bias levels.}
\label{fig:blastnet_channel_wise}
\end{figure}
\begin{figure}
    \centering
    \includegraphics[width=0.4\linewidth]{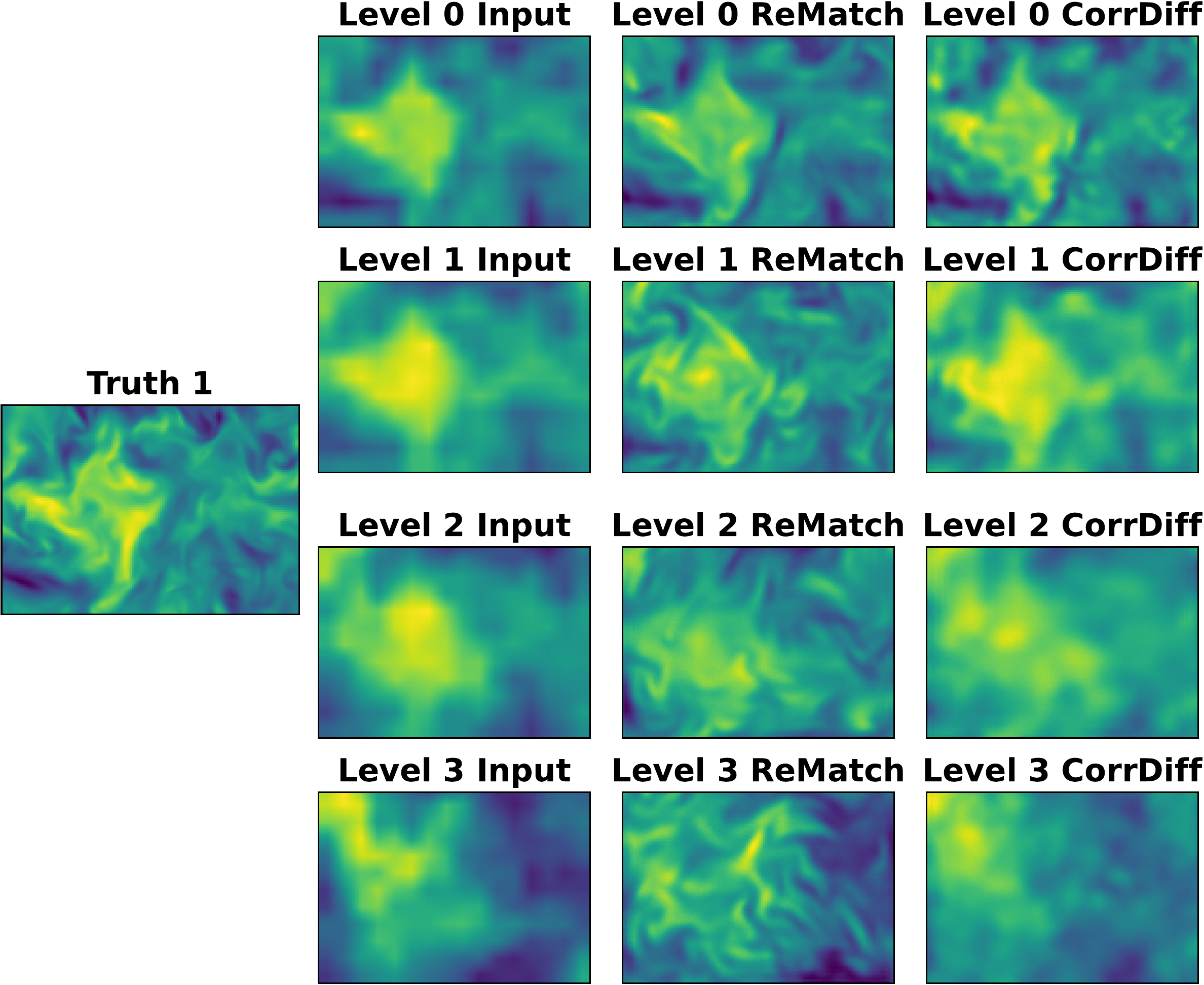}
    \caption{\textbf{Qualitative comparison across LR input bias levels.} Rows correspond to LR bias levels. Columns show the GT HR field, LR input, ReMatch and CorrDiff prediction.}
    \label{fig:blastnet_level_comparison}
\end{figure}

\subsection{{Calibration set sensitivity experiment}}
\label{app:calibration_sensitivity}
ReMatch relies on a held-out calibration set to estimate the target residual distribution. In our main experiments, we use a 2:1 train-to-calibration split. To examine the sensitivity of ReMatch to the calibration set size, we vary the train-to-calibration ratio while keeping the overall available training data fixed. Specifically, we consider ratios of 1:1, 2:1, 5:1, 10:1, and 100:1. CorrDiff corresponds to the special case with no calibration set and therefore no residual distribution matching.

We conduct this experiment on the BLASTNet dataset with bias level 2. For each train-to-calibration ratio, ReMatch is trained using the corresponding split and evaluated over 20 test subsets, each containing 100 samples. Table~\ref{tab:calibration_shift} reports the RMSE and SSR averaged across four channels. Overall, increasing the calibration set size tends to improve SSR while maintaining comparable or improved RMSE. Figure~\ref{fig:calibration_shift_ux} further shows the RMSE and SSR for the UX channel, with error bars indicating the standard deviation across the 20 test subsets.

A notable observation is the substantial drop in SSR for the 100:1 split. ReMatch transports training residuals toward the residual distribution estimated from the calibration set. Although the transport does not rely on one-to-one matching and instead uses a weighted combination of the top-$k$ transport components, an extremely small calibration set provides only a sparse representation of the target residual distribution. Consequently, the transported residuals can become concentrated around a limited set of calibration residuals, resulting in reduced predictive spread and lower SSR. These results suggest that the calibration set should be sufficiently large to represent the target residual distribution. However, allocating more samples to calibration leaves fewer samples for regression training and increases the computational cost of optimal transport. Therefore, the calibration split introduces a practical trade-off between regression accuracy, residual calibration, and computational cost.

\begin{table}[t]
\centering
% \resizebox{\linewidth}{!}{
\begin{tabular}{l|cccccc}
\toprule
 & CorrDiff & 100:1 & 10:1 & 5:1 & 2:1 & 1:1 \\
\midrule
RMSE$\downarrow$ & 0.1757&0.17370&0.17491&0.170417&0.17303&0.17190 \\
SSR$\uparrow$ &0.722&0.530632&0.73590&0.86210&0.84919&0.9038 \\
\bottomrule
\label{tab:calibration_shift}
\end{tabular}
%}
\caption{Sensitivity of ReMatch to the train-to-calibration split on the BLASTNet dataset on bias level 2. RMSE and SSR are averaged across four output channels over 20 test subsets of 100 samples each. CorrDiff corresponds to the case without residual distribution matching.}\label{tab:calibration_shift}
\end{table}

\begin{figure}
    \centering
    \includegraphics[width=0.5\linewidth]{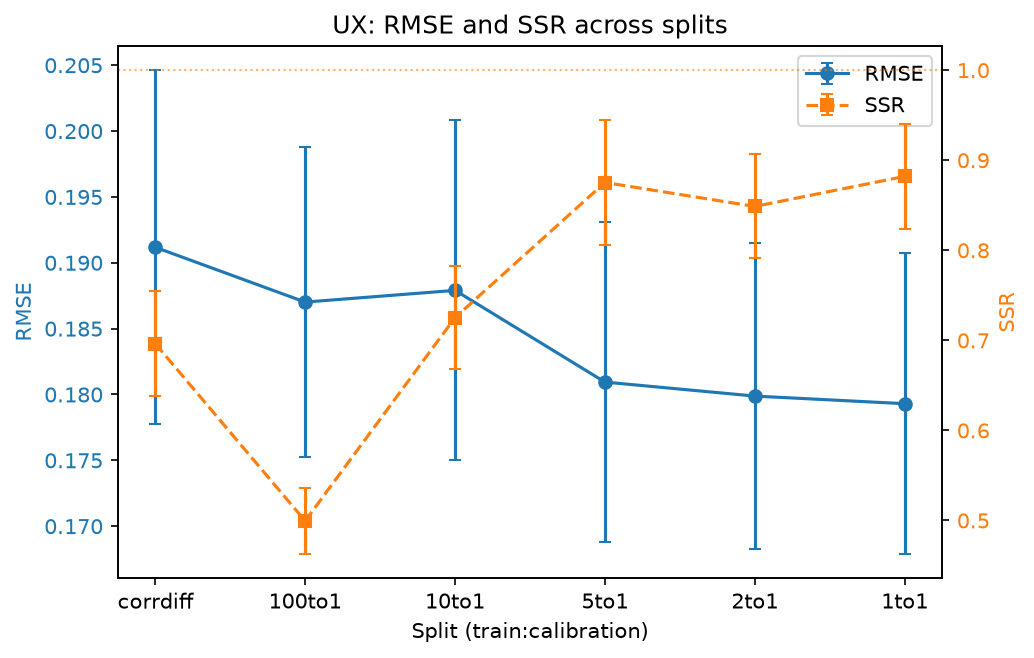}
    \caption{Sensitivity of ReMatch to the train-to-calibration split for the UX channel of the BLASTNet dataset. The left and right y-axes show RMSE and SSR, respectively. Each point reports the mean over 20 test subsets of 100 samples, with error bars indicating one standard deviation.}
    \label{fig:calibration_shift_ux}
\end{figure}

\subsection{ERA5-HRRR Dataset Details}
\label{appendix:hrrr_era5_details}
We use the \(u\)- and \(v\)-components of wind at five lower-stratospheric pressure levels: 50, 75, 100, 125, and 150 hPa. The experiments are conducted over a \(500\,\mathrm{km} \times 500\,\mathrm{km}\) region in the northeastern CONUS, defined by \(\mathrm{lat}\in[36.0795, 41.8914]^\circ\) and \(\mathrm{lon}\in[-82.5501, -75.2087]^\circ\).

\paragraph{ERA5~\cite{hersbach2020era5}} We downloaded \(25 \times 25\,\mathrm{km}\) scale ERA5 reanalysis data for every hour from 2018-2021. We note that ERA5 only has 70 hPa instead of 75 hPa pressure level data, but we consider them as the same channel since it does not detract from our setting of input/output mismatch. We use the direct low-resolution input to train the SwinIR models, but to create the low-resolution input for all our other experiments, we follow the default PhysicsNeMo CorrDiff pipeline and use bilinear interpolation to upsample data to match the \(3 \times 3\,\mathrm{km}\) resolution of HRRR data. We used 2018-2020 as our train set and 2021 as our test set. For methods that need a calibration set (UC, ReMatch), we used 2018-2019 as the train set, 2020 as the calibration set, and 2021 as the test set.

\paragraph{HRRR~\cite{dowell2022high}} We used \(3 \times 3\,\mathrm{km}\) scale HRRR pressure data for every hour from 2018-2021 using lead times and forecasts of up to 6 hours to fill gaps.

\begin{figure}[h]
\centering
\includegraphics[width=1\linewidth]{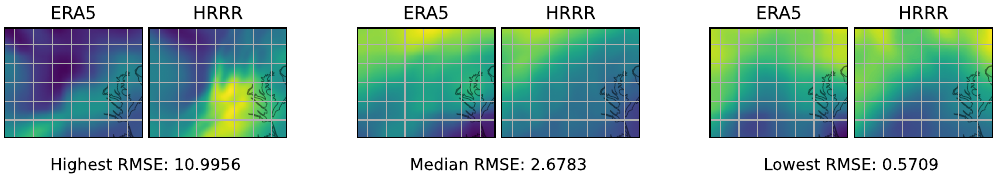}
\caption{\textbf{Downscaling bias in ERA5 and HRRR has a positive relationship with RMSE.} We show three samples of upsampled ERA5 data and HRRR 100hPa $u$-wind from our 2018-2020 training set. The left sample with high RMSE shows a clear mismatch in the center while the right sample with low RMSE generally looks the same, indicating that RMSE is a good estimator of input/output mismatch.}
\label{fig:era_hrrr_gap}
\end{figure}

To create the combined dataset, we also added four sinusoidal-embedding temporal input channels for the day and month. Fig.~\ref{fig:era_hrrr_gap} indicates the main challenge of the real-world downscaling task with atmospheric data. There is an inherent mismatch between the LR and HR data over the same region since they come from two different sources. We match the spatial dimensions of the two datasets but do not explicitly reproject ERA5 onto the native HRRR grid. Consequently, the ERA5–HRRR residual can contain a systematic contribution associated with differences in spatial sampling and grid geometry, differences in model physics and data assimilation from cross-sensor sources, systematic pressure-level misalignment for the 75 hPa channel, and stochastic finer scale atmospheric variability. 

Table~\ref{tab:era_hrrr_channelwise_stats} shows higher bias levels for higher pressure levels, especially in the v-wind channel. We hypothesize this is due to higher pressure levels being closer to the jet stream (250 hPa) which has strong wind magnitudes that could cause more sensor inaccuracies.
\begin{table}[t]
\centering
\caption{\textbf{Channel-wise statistics of downscaling bias.}}
\label{tab:era_hrrr_channelwise_stats}
%\resizebox{\linewidth}{!}{
\begin{tabular}{l|ccccccccccc}
\toprule
Channel & 50u & 50v & 75u & 75v & 100u & 100v & 125u & 125v & 150u & 150v \\
\midrule
Mean & 1.99 & 2.59 & 2.77 & 3.52 & 2.68 & 4.27 & 3.08 & 5.19 & 3.62 & 6.19 \\
Std. Dev. & 0.81 & 1.39 & 1.18 & 1.93 & 1.08 & 2.21 & 1.27 & 2.56 & 1.55 & 2.30 \\
% 50u & 1.99 & 0.81\\
% 50v & 2.59 & 1.39\\
% 75u & 2.77 & 1.18\\
% 75v & 3.52 & 1.93\\
% 100u & 2.68 & 1.08 \\
% 100v & 4.27 & 2.21\\
% 125u & 3.08 & 1.27\\
% 125v & 5.19 & 2.56\\
% 150u & 3.62 & 1.55 \\
% 150v & 6.19 & 2.30\\
\bottomrule
\end{tabular}
%}
\par
\bigskip
\end{table}

\subsection{Implementation Details}
\label{appendix:baselines}
\label{app:experiment_details}
ReMatch is not tied to a specific architecture. It only assumes a deterministic mean predictor followed by a stochastic residual generator, so the mean predictor can be instantiated with any suitable pointwise estimator and the residual stage with any stochastic or ensemble-based generative model. With this decomposition, we can model predictable biases from the ERA5 input with the deterministic regression component, including systematic or state-dependent differences between the two datasets. We can model variability that is not determined by the available low-resolution input with the stochastic residual diffusion component, allowing the model to represent multiple plausible HRRR realizations for a given ERA5 input. 
In our experiments, we use U-Net regression and SwinIR~\cite{liang2021swinir} as deterministic mean predictors, trained with \(\ell_2\) and \(\ell_1\) losses, respectively. For stochastic residual prediction, we use EDM-style diffusion preconditioning and sampling~\cite{karras2022elucidating} with a U-Net-based score network~\cite{song2020score}. The low-resolution input \(x\) and mean field \(\mu\), either transported pseudo-mean from training set \(\tilde{\mu}_i\) or mean predictor output in calibration set \(\mu_\phi(x_j)\), are concatenated channel-wise as the conditioning input. Other stochastic residual generators can be used in the same framework.

Unless otherwise noted, all UNet-based regression and diffusion models are trained under the same implementation settings to ensure a fair comparison. We use data-parallel training on four NVIDIA A100 GPUs with a per-GPU batch size of 32 for experiments, in which training takes about 8-10 hours. All regression UNets and diffusion models are trained for 8M processed samples, following the training-duration convention used in the CorrDiff implementation at \url{https://github.com/NVIDIA/physicsnemo/blob/main/examples/weather/corrdiff/README.md}.

For the CorrDiff diffusion backbone, we instantiate \texttt{EDMPrecondSuperResolution} with a \texttt{SongUNetPosEmbd} denoising network. The \texttt{mini} model-size preset controls the capacity of the underlying U-Net through three parameters: \texttt{model\_channels}, \texttt{channel\_mult}, and \texttt{attn\_resolutions}. Specifically, \texttt{model\_channels}=64 sets the base channel width; \texttt{channel\_mult}=\([1,2,2]\) defines a three-level encoder--decoder hierarchy with feature widths \(64,128,128\); and \texttt{attn\_resolutions}=\([16]\) inserts self-attention blocks at the \(16\times16\) feature-map resolution. 
% During generation, we use 12 ensemble members and 4 NVIDIA A100 GPUs. 
All other architectural choices follow the default PhysicsNeMo CorrDiff settings.

\paragraph{CorrDiff~\cite{mardani2025residual}}\label{appendix:corrdiff}A two-stage mean--residual baseline consisting of a regression mean predictor followed by a diffusion model for stochastic residual prediction. 
% We also evaluate \textbf{CorrDiff\textsubscript{m}}, which uses a reduced-capacity mean predictor to mitigate residual target distribution shift. We set \texttt{model\_channels}=\(16\), \texttt{channel\_mult}=\([1,2]\), and \texttt{attn\_resolutions}=\([4]\) and 3M total processed samples.
\paragraph{UNet~\cite{song2020score}} A deterministic regression baseline to map the LR input directly to the HR output.
\paragraph{Conditional Diffusion Model (CDM)} A diffusion model conditioned on the LR input without an upstream mean predictor.
\paragraph{CFG~\cite{ho2022classifier}}
Classifier-free guidance applied only to the residual diffusion model in CorrDiff to sharpen residual samples and mitigate small-scale underdispersion. We implement classifier-free guidance (CFG) on top of the CorrDiff architecture as a baseline. CFG allows conditional sampling without training an additional classifier. During training for diffusion, we zero out the conditioning on the LR input 10\% of the time so the model learns both conditional and unconditional inputs. During inference, we use 18 denoising steps in the stochastic sampler with a cosine guidance weighting schedule from $w=$ 0 to 8. Comparing this to fixed weights from 0 to 8, we found that scheduling increasing guidance helped maintain RMSE skill while improving spread. The sample at each time step is then a linear combination of the conditional and unconditional output. This encourages the model to learn features induced specifically by the conditioning.
\paragraph{Finetuning~\cite{couairon2026arches}.}
We construct a fine-tuning baseline following the held-out residual adaptation strategy of ArchesWeatherGen~\cite{couairon2026arches}. We first train the deterministic regression model and residual diffusion model on the 2018--2019 training set for 8M processed samples each. We then freeze the regression model and fine-tune only the residual diffusion model for an additional 1M processed samples on the 2020 calibration set. Unlike the original method, we do not use inference-time noise rescaling and retain the same diffusion architecture as the other baselines.
\paragraph{UC~\cite{kendall2017uncertainty}} A CorrDiff variant that conditions the residual diffusion model on an uncertainty estimate of the regression mean, predicted by a lightweight UNet. In addition to UC\textsubscript{$\mu$} that was discussed in Section~\ref{sec:experiments}, we also tested using error quantiles as an input, which we call UC\textsubscript{q}. Specifically, \(\mathrm{UC}_{\mu}\) conditions on the predicted MAE, and \(\mathrm{UC}_{q}\) conditions on the predicted lower and upper error quantiles, \(q_5\) and \(q_{95}\).
To construct these baselines, we first train a deterministic mean predictor on the 2018--2019 training set. We then train a separate uncertainty estimator on the 2020 calibration set to predict either the MAE or the \(q_5/q_{95}\) error quantiles of the mean predictor. The uncertainty estimator uses the same UNet as the regression mean and is trained for 8M samples. Using the calibration set as a proxy for the test regime allows the uncertainty network to learn held-out error statistics, rather than relying on training-time errors that may be overconfident.
% For training the uncertainty network, we use a reduced-capacity regression network from \textbf{CorrDiff\textsubscript{m}} at 3M processed samples and we train the uncertainty network on the 2020 validation set for 8M processed samples.
% Using the calibration set as proxy of test set helps the uncertainty network to fit in test-time error instead of when it can behave overconfident during train time.
\begin{itemize}
\item  \(\mathrm{UC}_{\mu}\): The network predicts a scalar MAE value of the residuals between the ground truth and the mean predictor. Uses L2 loss.
\item \(\mathrm{UC}_{q}\): The network predicts scalar quantile values for Q5 and Q95. Uses Pinball loss.
\end{itemize}
After training the mean predictor and uncertainty estimator sequentially, we train the residual diffusion model conditioned on the LR input, the predicted mean, and an uncertainty measure. 
During diffusion training, we use the ground-truth uncertainty measure computed from the mean-prediction error, while at test time we replace it with the output of the uncertainty estimator. 
This choice stabilizes diffusion training by providing accurate uncertainty conditioning during optimization, while preserving a fully predictive pipeline at inference. 
The residual diffusion model is trained on the combined 2018--2020 data, including both the training and calibration splits.

Additionally, Figure ~\ref{tab:uc_q_results} displays the results of UC\textsubscript{q} for the same 600 test samples as Table ~\ref{tab:main_results}. While it has a slightly better RMSE, LPIPS, and ACC compared to the UC\textsubscript{$\mu$} variant, it performs worse on calibration metrics, which are essential for good probabilistic models.

% \begin{figure}[htbp]
% \caption{\textbf{UC\textsubscript{q} quantitative and qualitative results on HRRR-ERA5 pressure-level super-resolution over all channels.}  We show downscaled 50 hPa v-wind fields on 2021-06-05 10:00:00 with the 7th ensemble member for consistent comparison to other methods.}    
% \label{fig:uc_q_results}
%     % Left side: The Image Minipage
%     \begin{minipage}[c]{0.55\textwidth}
%         \centering
%         %\includegraphics[trim=15.5cm 0cm 14.5cm 6cm, clip, width=\linewidth]{fig/compare_all/comparison_5_ucq_time318_150v_ens2_2021-07-06.pdf}
%         \includegraphics[trim=22cm 0cm 14.5cm 6cm, clip, width=0.6\linewidth]{fig/compare_all/comparison_5_ucq_time318_150v_ens2_2021-07-06.pdf}
%         \label{fig:uc_q_qualitative_results}
%     \end{minipage}
%     %\hfill % Adds horizontal space to push the table to the right
%     % Right side: The Table Minipage
%     \begin{minipage}[c]{0.35\textwidth}
%         \centering
%         \begin{tabular}{l|c}
% \toprule
% Metric 
% & UQ\textsubscript{q} \\
% \midrule

% {RMSE} $\downarrow$
% & 1.8551 \\

% {SSR} $\rightarrow$ 1
% & 0.5950 \\

% {CRPS/MAE} $\downarrow$
% & 1.1946 \\

% {LPIPS} $\downarrow$
% & 0.0842 \\

% {ACC} $\rightarrow$ 1
% & 0.8952 \\

% \bottomrule
% \end{tabular}
%         \label{tab:uc_q_results}
%     \end{minipage}
% \end{figure}

\begin{table}[t]
\centering
\caption{\textbf{UC\textsubscript{q} quantitative results on HRRR-ERA5 pressure-level super-resolution over all channels.}}
\label{tab:uc_q_results}
\begin{tabular}{l|ccc}
\toprule
Metric 
& UQ\textsubscript{$q$} & UQ\textsubscript{$\mu$} & Rematch\textsubscript{U}\\
\midrule

{RMSE} $\downarrow$
& 1.8551 & 1.8584 & 1.8102 \\

{SSR} $\rightarrow$ 1
& 0.5950 & 0.8745 & 0.6347 \\

{CRPS/MAE} $\downarrow$
& 1.1946 & 1.1545 & 1.1567 \\

{LPIPS} $\downarrow$
& 0.0842 & 0.0871 & 0.0841\\

{ACC} $\rightarrow$ 1
& 0.8952 & 0.8926 & 0.9017 \\

\bottomrule
\end{tabular}
\end{table}

\paragraph{SwinIR~\cite{liang2021swinir}}
A deterministic transformer-based super-resolution baseline, which directly maps the LR input to the HR field. We implement a SwinIR-based super-resolution baseline following the network configuration used in the original SwinIR model at \href{https://github.com/JingyunLiang/SwinIR}{https://github.com/JingyunLiang/SwinIR}. 
For the network architecture, we use the main SwinIR configuration from the original paper: the embedding dimension is set to 180, the model contains six Residual Swin Transformer Blocks (RSTBs), and each RSTB contains six Swin Transformer layers. 
The number of attention heads is set to 6 in each RSTB, the MLP ratio is set to 2, and the residual connection inside each RSTB uses a single \(3 \times 3\) convolution. 
We also use the standard multi-step PixelShuffle reconstruction head with an intermediate feature dimension of 64. 
The only modification from the original SwinIR parameter setting is the window size. 
While the original model uses a window size of 8 for super-resolution, we set the window size to 7 so that the self-attention windows exactly divide the low-resolution input size of \(21 \times 21\).
The model is trained using a \(\ell_1\) loss between the predicted and ground-truth high-resolution fields. 
We train the SwinIR baseline for 100k processed samples, at which point the validation loss is observed to have converged.

\paragraph{Conv-FNO~\cite{wen2022u}}  A deterministic neural-operator baseline, which combines Fourier neural operator layers with convolutional blocks for direct LR-to-HR super-resolution.
We implement an FNO-based super-resolution baseline using a 2D FNO model from NVIDIA's PhysicsNeMo at \url{https://github.com/NVIDIA/physicsnemo/blob/main/physicsnemo/models/fno/fno.py} and modifying the architecture and parameters according to the BlastNet Conv-FNO model at \url{https://github.com/blastnet/blastnet2_sr_benchmark/blob/ae2a3a6b1c1c8fb9dcf478aa47ee254159c42f1b/models/convfno.py#L12}, that modified a 3D FNO model to be used for super-resolution.

Our main modifications were replacing FNO's decoder with an Upsampler block and adding a ResidualBlock to the spectral layers in FNO2DEncoder. We use 34 features, 34 latent channels, 32 spectral and residual layers, kernel size of 3, 2 modes, a padding of 6, and a 8x upscaling factor. These parameters were chosen to match the implementation in BlastNet to remain similar to an established super-resolution baseline. The model was trained using MSE loss between the predicted and ground-truth high-resolution fields. We train the Conv-FNO baseline for 8 million processed samples, similar to the CorrDiff regression capacity.

\newpage
\subsection{Additional Analysis on Real-World Wind Field Downscaling Task}
\label{appendix:more_results}
\subsubsection{Method Comparisons}
We provide additional qualitative examples of downscaling results for all baseline models and the proposed method.
\begin{figure}[h]
    \centering
    \begin{subfigure}{1\linewidth}
        \centering
        \includegraphics[width=1\linewidth]{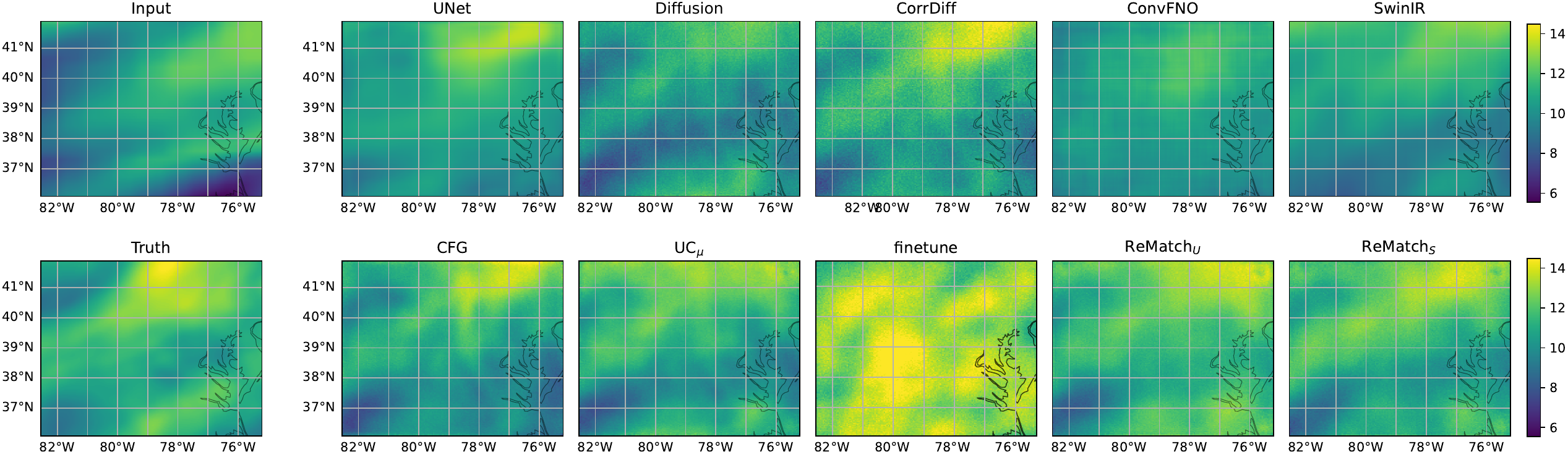}
        \caption{Qualitative comparison of downscaled 125 hPa \(u\)-wind fields on 2021-07-20 00:00:00.}
        \label{fig:method_all_75u}
    \end{subfigure}
    \\
    \begin{subfigure}{1\linewidth}
        \centering
        \includegraphics[width=1\linewidth]{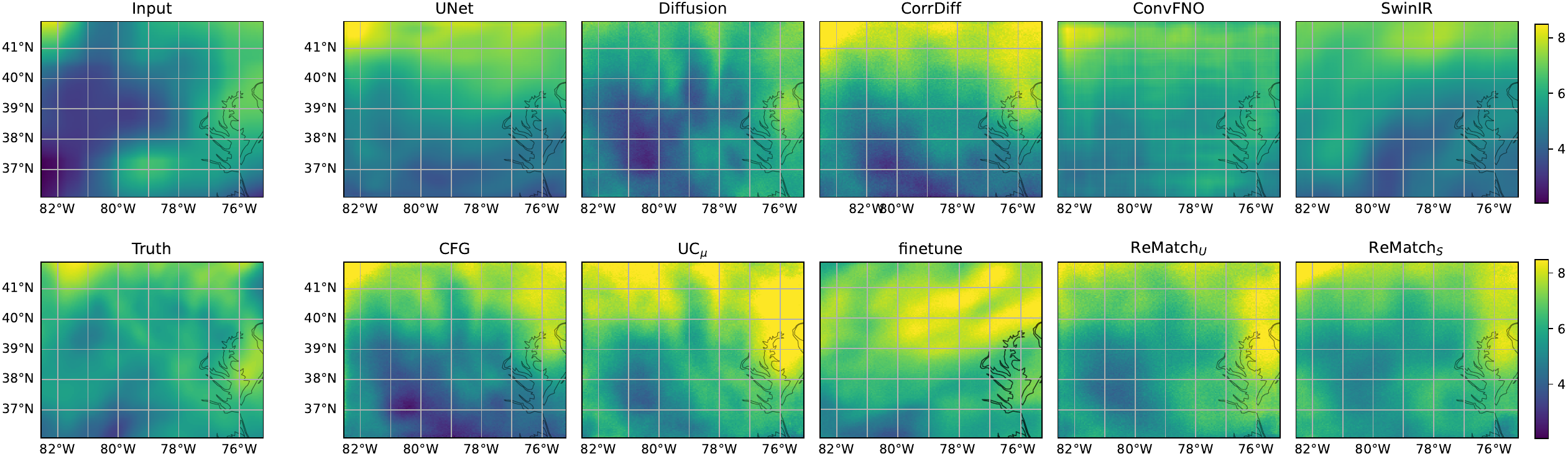}
        \caption{Qualitative comparison of downscaled 75 hPa \(v\)-wind fields on 2021-09-14 01:00:00.}
        \label{fig:method_all_100v}
    \end{subfigure}
    \\
    \begin{subfigure}{1\linewidth}
        \centering
        \includegraphics[width=1\linewidth]{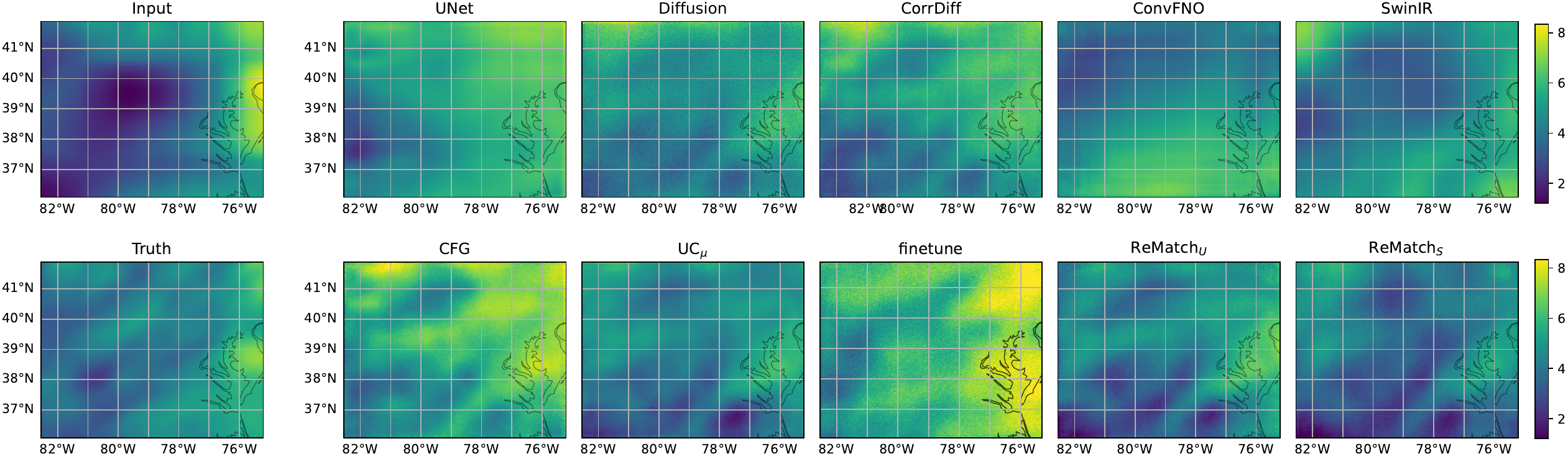}
        \caption{Qualitative comparison of downscaled 125 hPa \(u\)-wind fields on 2021-07-06 19:00:00.}
        \label{fig:method_all_125u}
    \end{subfigure}
    \\
    \begin{subfigure}{1\linewidth}
        \includegraphics[width=1\linewidth]{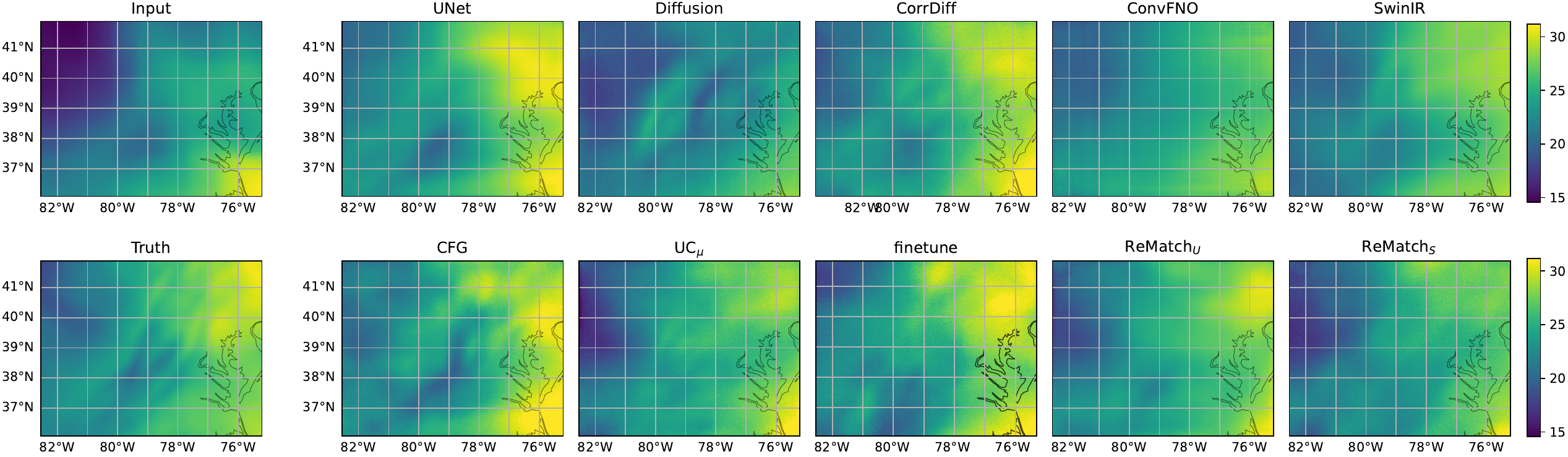}
        \caption{Qualitative comparison of downscaled 150 hPa \(u\)-wind fields on 2021-10-19 7:00:00.}
        \label{fig:method_all_150u}
    \end{subfigure}
    \\
    \caption{Qualitative comparison of all baseline and proposed methods across randomly selected channels and dates.}
    \label{fig:all_method_exmaples}
\end{figure}

\newpage
\subsection{Ensemble Samples Comparison}
We further analyze individual ensemble samples. Figure~\ref{fig:compare_all} shows downscaling results for the 50 hPa $v$-wind field on 2021-09-05 at 10:00:00. In Figure~\ref{fig:ensembles}, we show all 12 ensemble samples of the final predictions, \(\mu_\phi(x) + r_\theta(x,\mu_\phi)\), together with their predicted residuals \(r_\theta(x,\mu_\phi)\).

\begin{figure}[h]
    \centering
    \begin{subfigure}{1\linewidth}
        \centering
        \includegraphics[width=\linewidth]{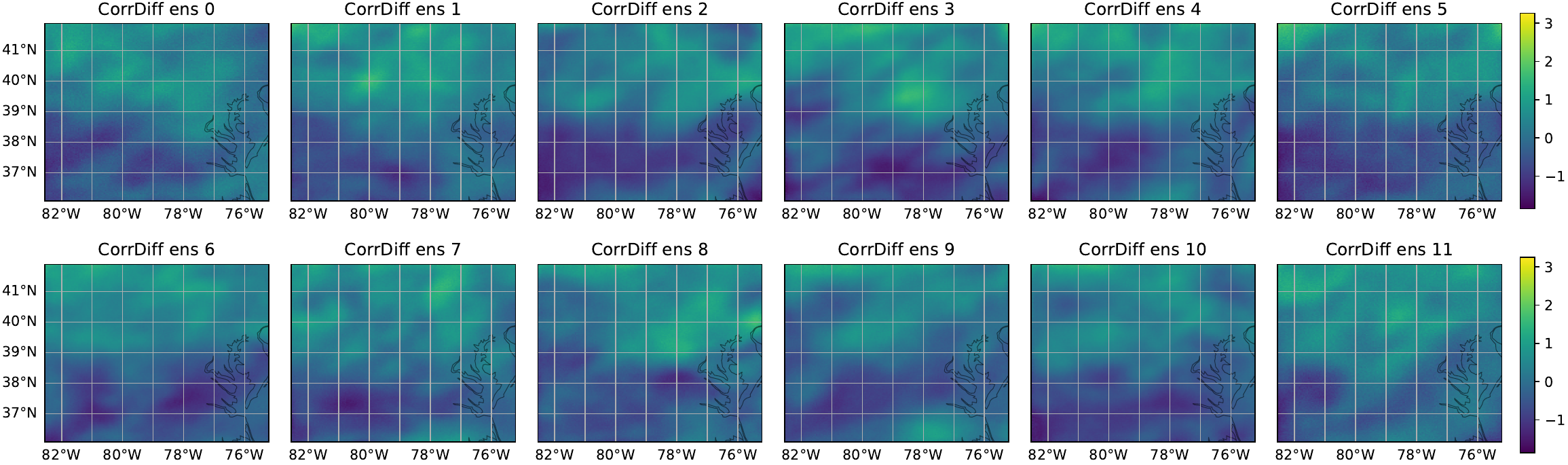}
        \caption{CorrDiff all prediction ensemble samples.}
        \label{fig:corrdiff_predictions}
    \end{subfigure}
    \\
    \begin{subfigure}{1\linewidth}
        \centering
        \includegraphics[width=\linewidth]{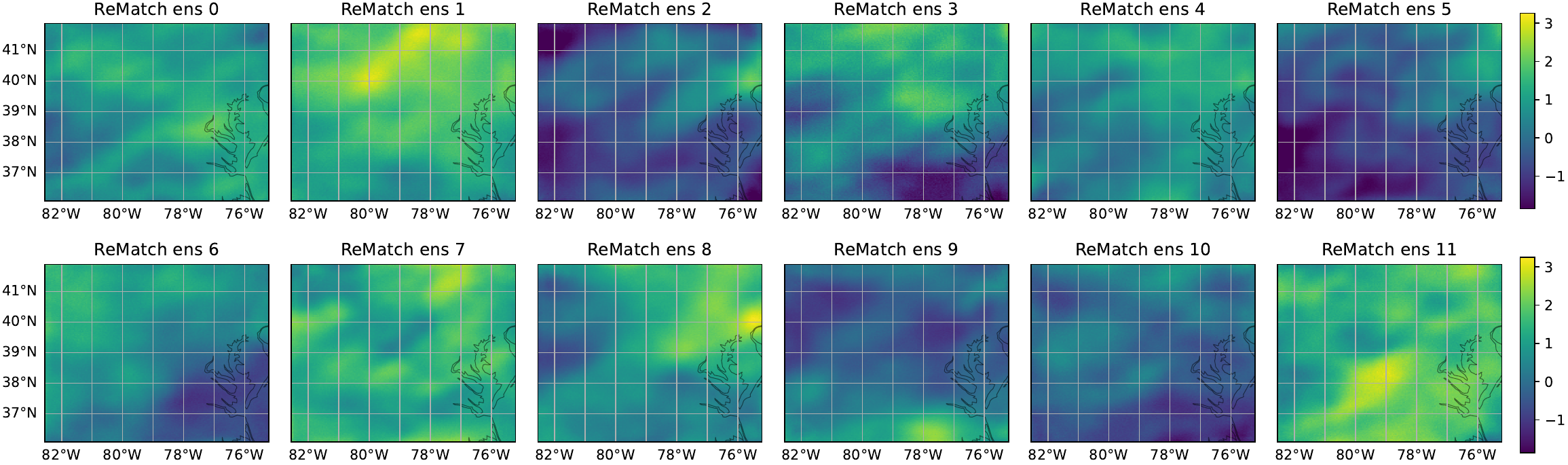}
        \caption{ReMatch\textsubscript{U} all prediction ensemble samples.}
        \label{fig:rematch_predictions}
    \end{subfigure}
    \\
    \begin{subfigure}{1\linewidth}
        \centering
        \includegraphics[width=\linewidth]{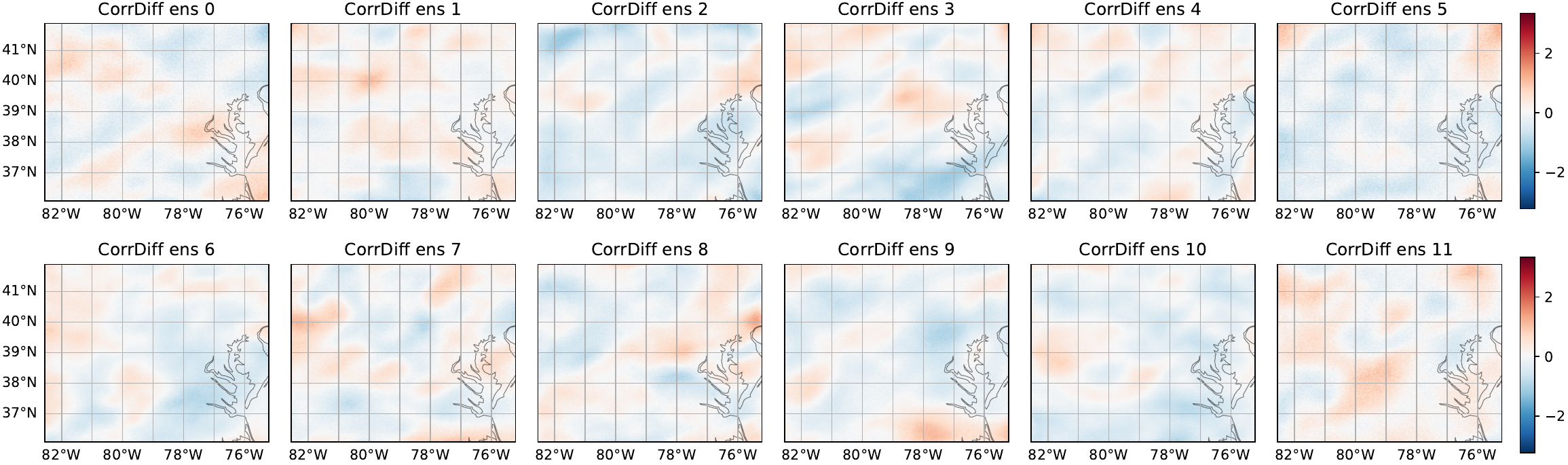}
        \caption{CorrDiff all residual ensemble samples.}
        \label{fig:corrdiff_residuals}
    \end{subfigure}
    \\
    \begin{subfigure}{1\linewidth}
        \centering
        \includegraphics[width=\linewidth]{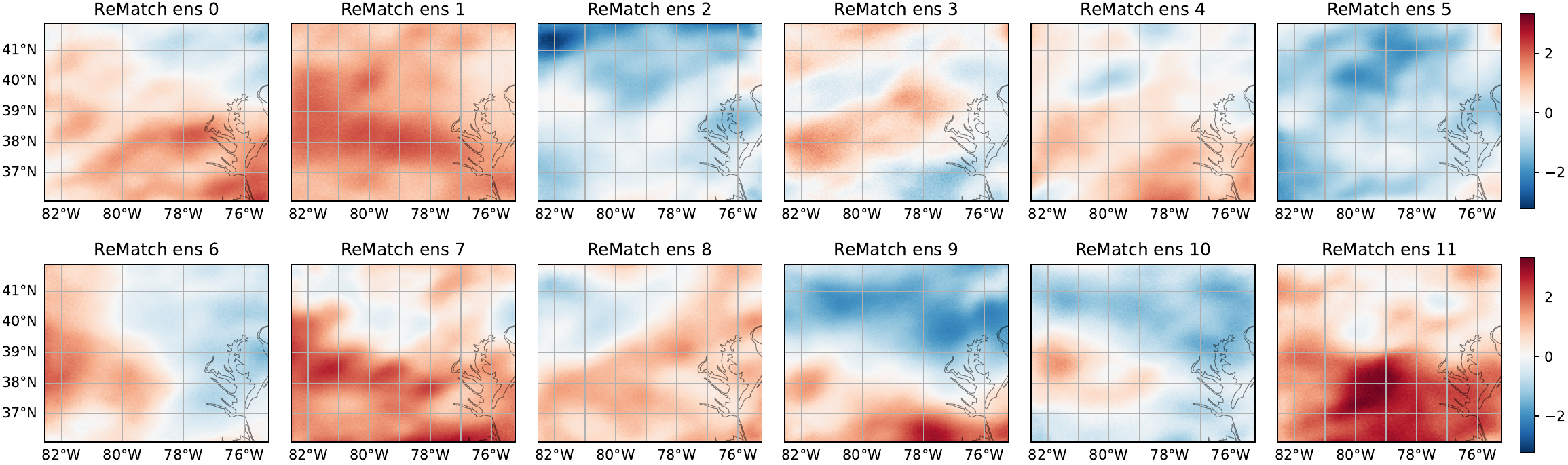}
        \caption{ReMatch\textsubscript{U} all residual ensemble samples.}
        \label{fig:rematch_residuals}
    \end{subfigure}
    \caption{Ensemble predictions from Fig.~\ref{fig:compare_all}.}
    \label{fig:ensembles}
\end{figure}
Consistent with prior observations~\cite{mardani2025residual, fotiadis2025adaptive} and our quantitative results, CorrDiff produces underdispersed predictions with a low SSR. This behavior is visible in Figure~\ref{fig:ensembles}: CorrDiff ensemble samples remain highly similar across members, as shown in Figure~\ref{fig:corrdiff_predictions}, and their residual magnitudes are relatively small, as shown in Figure~\ref{fig:corrdiff_residuals}, compared with ReMatch residuals in Figure~\ref{fig:rematch_residuals}. This suggests that CorrDiff fails to generate sufficiently corrective residuals under residual misspecification. Given the error-prone mean prediction shown by the UNet result in Figure~\ref{fig:compare_all}, the residual model should correct large-scale mean errors rather than merely add high-frequency details. However, CorrDiff largely fails to do so, whereas ReMatch produces more diverse and corrective residuals.

\subsection{ACC Score analysis}
\label{appendix:acc_score}
\begin{figure}[h]
\centering
\includegraphics[width=0.7\linewidth]{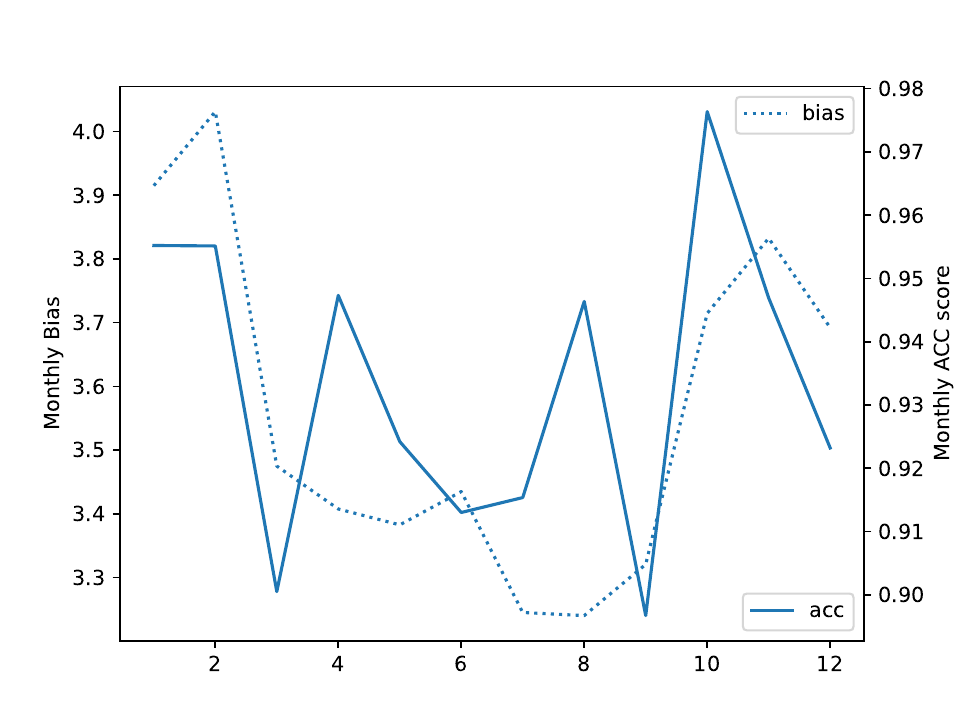}
\caption{\textbf{Comparison of ReMatch\textsubscript{U} ACC and input/output bias per month.} ACC score per month calculated for 2021 and average bias (RMSE) per month calculated over 2018-2021.}
\label{fig:accvrmse_permonth}
\end{figure}
With the ACC Score calculated across 2021, we can compare how spatially similar ReMatch\textsubscript{U} anomalies and HRRR anomalies are with the estimated climatological mean from 2018-2020 HRRR data. Anomalies are when the atmospheric variable deviates from the expected historical value at a certain time and location. We note that having a longer range of years to estimate the climatological mean would likely lead to more accurate estimates of anomalies. Fig.~\ref{fig:accvrmse_permonth} shows a monthly breakdown of downscaling bias and captures how our model adapts to the bias gap between ERA5 and HRRR. In winter months the bias is higher, but our model maintains high ACC scores, so ReMatch\textsubscript{U} is able to capture anomaly structure.

\subsection{Kinetic Energy Analysis}
\begin{figure}
    \centering
    \includegraphics[width=1.\linewidth]{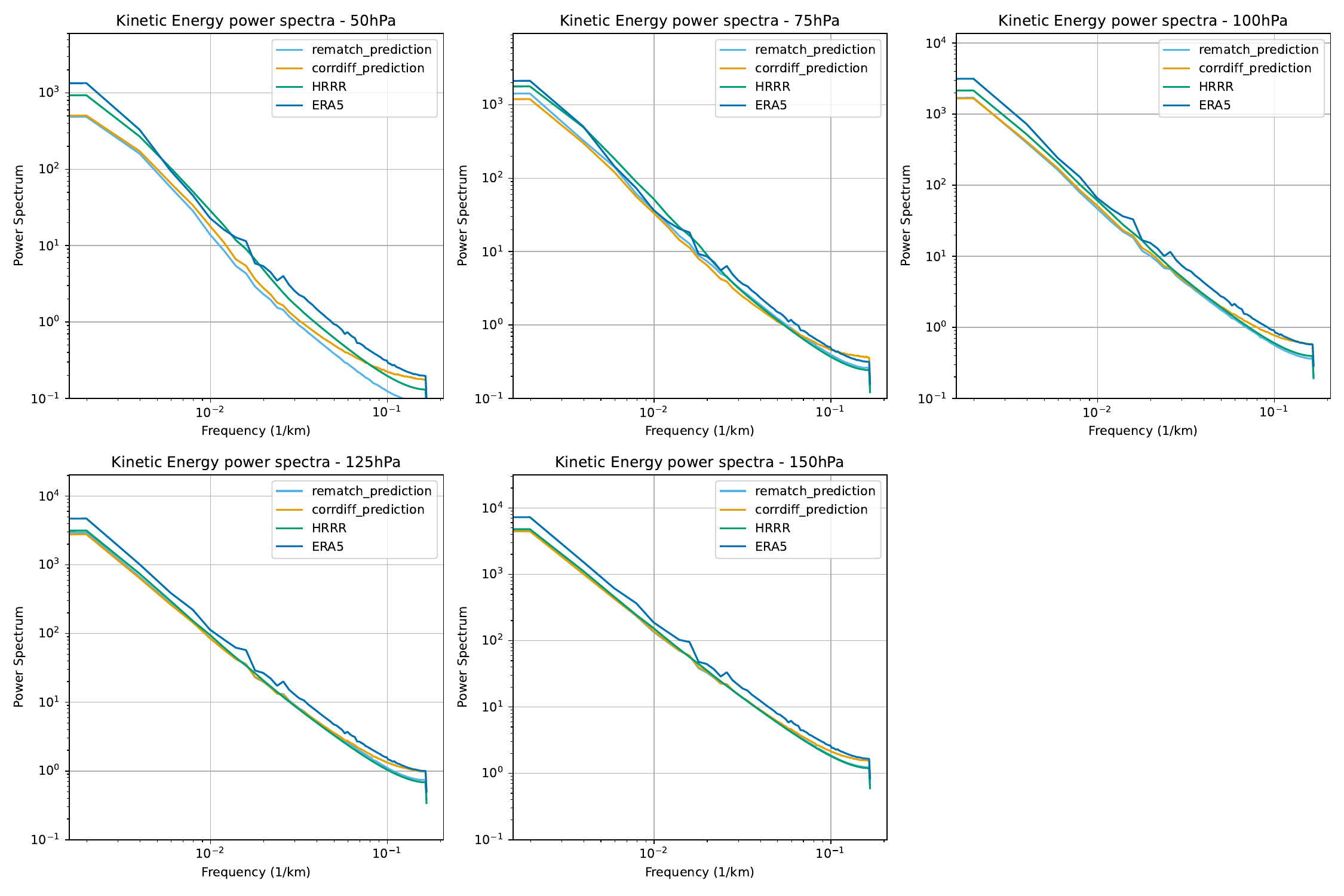}
    \caption{\textbf{Power spectra for kinetic energy in all pressure levels}}
    \label{fig:ke_spectrum}
\end{figure}
Figure~\ref{fig:ke_spectrum} compares the 1D kinetic energy spectra of the ground-truth HRRR field, the ERA5 input, CorrDiff predictions, and ReMatch predictions. This diagnostic measures how kinetic energy is distributed across spatial scales: low wavenumbers correspond to large-scale structure, whereas high wavenumbers correspond to finer-scale variability. A successful downscaling model should therefore not only reduce pointwise error, but also recover the target spectrum over the relevant range of wavenumbers.

To compute this quantity, we form the horizontal kinetic energy from the wind components and evaluate its one-dimensional Fourier spectrum. For each sample, we compute the Fourier coefficients of the horizontal wind components \(u\) and \(v\), and form
\[
E(k) \propto \frac{1}{2}\Big(|\hat u(k)|^2 + |\hat v(k)|^2\Big).
\]
We average over the remaining spatial dimension and over samples. For stochastic methods, the spectrum is computed separately for each ensemble member and then averaged across ensemble members.

The figure shows that ERA5 overestimates kinetic energy at all wavenumbers relative to HRRR, indicating a scale-dependent bias. ReMatch removes a substantially larger fraction of this excess power and tracks the HRRR spectrum more closely over a broad range of wavenumbers. The gap between CorrDiff and ReMatch is most visible at higher wavenumbers, where CorrDiff preserves more of the excessive energy inherited from ERA5, whereas ReMatch remains closer to the HRRR target. 
% The same trend is observed across all pressure levels; we omit the remaining spectra to avoid repetition.

This pattern differs from surface-level wind fields shown in \cite{mardani2025residual}, where HRRR often contains much stronger high-frequency variability due to land--sea contrast, topography, and other surface effects. At pressure levels, the flow is smoother, so even the HRRR target has relatively weak high-frequency energy. In this regime, the main issue is therefore not only missing fine-scale structure, but a broader mismatch in scale-dependent energy.

\subsection{Uncertainty Conditioned Experiments}
\label{appendix:uncertainty}
Previous work on image regression~\cite{angelopoulos2022image,kendall2017uncertainty} often learns pixelwise epistemic or aleatoric uncertainty jointly with the regression predictor. In our setting, we adapt this idea to the mean--residual CorrDiff architecture by training an auxiliary model to predict uncertainty about the upstream mean and using that quantity as an additional conditioning signal for residual diffusion. Empirically, however, dense pixelwise uncertainty maps did not learn meaningful or predictive structure, whereas low-dimensional scalar summaries of the regression error---such as MAE, \(Q_5\), \(Q_{90}\), and \(Q_{95}\)---were much easier to predict and matched the ground-truth values more reliably on the test set (Fig.~\ref{fig:uncertainty_scalarvpixel}).
We explored two other spatially focused methods called Moran's I (a local indicator of spatial autocorrelation) and VGGT~\cite{wang2025vggt} inspired confidence maps. While the trained networks provided more meaningful feature maps~\ref{fig:uncertainty_badpixelwise}, they ultimately did not perform as well as conditioning to diffusion.

We interpret this through the lens of residual target misspecification. Under downscaling bias, the regression error \(Y-\mu_\phi(X)\) is itself a predictor-defined correction target whose distribution shifts across regimes. Learning a full pixelwise uncertainty field therefore amounts to inferring a high-dimensional summary of a misspecified residual target from limited observable information, which is substantially more difficult than predicting a coarse scalar descriptor of its magnitude. This explains why scalar uncertainty signals can still provide useful information to the diffusion model, as reflected in the UC baselines, while dense pixelwise uncertainty prediction remains unreliable. More broadly, these results suggest that uncertainty prediction alone is not sufficient to resolve residual misspecification: compressing the error into low-dimensional summaries may help, but recovering its full spatial structure requires a more direct treatment of the shifted residual target itself.
\begin{figure}[h]
    \centering
    \includegraphics[width=\linewidth]{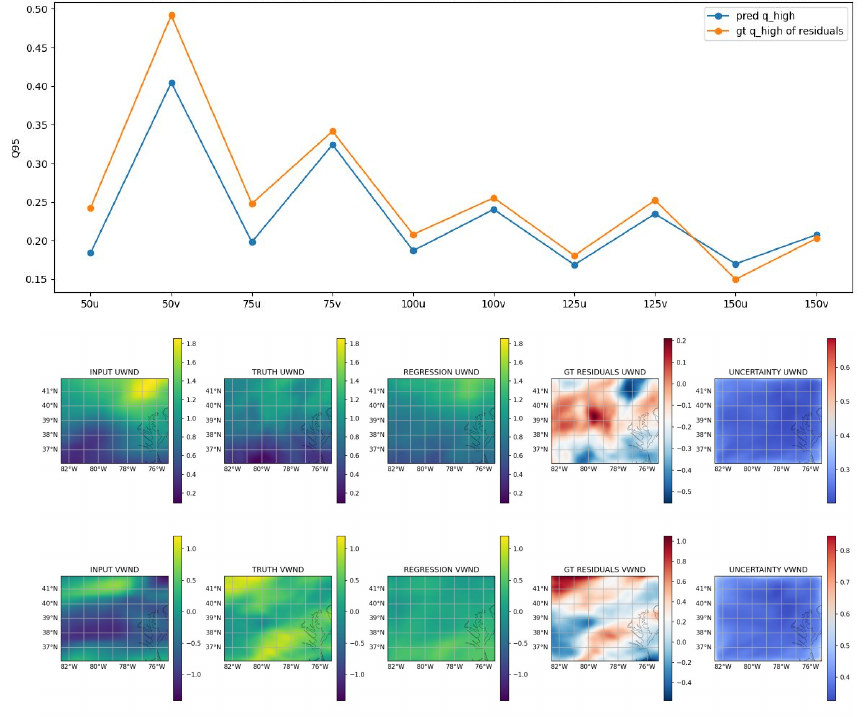}
        \caption{\textbf{Comparison of scalar and pixelwise uncertainty estimators.} Both networks predict q5 and q95 of the residuals per channel on the test set. The top plot is a scalar prediction while the bottom plot displays a pixelwise prediction at 2021-01-02 11:00:00 at 50 hPa $u$-wind field.}
        \label{fig:uncertainty_scalarvpixel}
\end{figure}

\begin{figure}[h]
    \centering
    \includegraphics[width=0.5\linewidth]{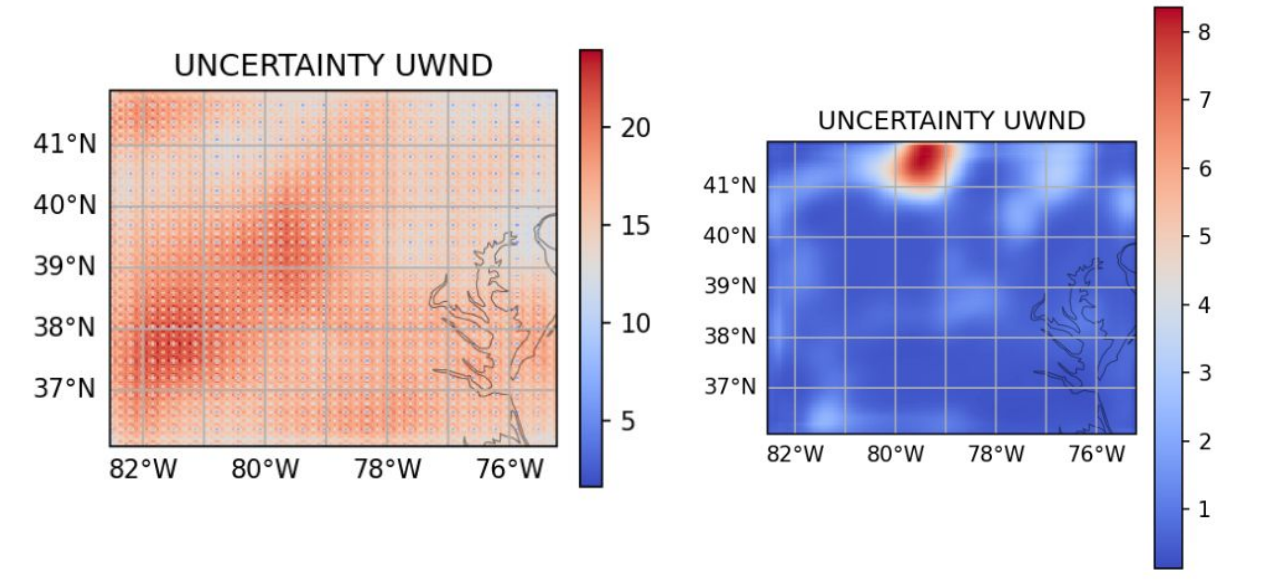}
        \caption{\textbf{Comparison of VGGT inspired (left) and Moran's I (right) pixelwise uncertainty networks.} Both samples are also from 2021-01-02 11:00:00 at 50 hPa $u$-wind field. The Moran's I uncertainty focuses more on "interest points" while the VGGT inspired confidence is more spatially aware. However, the VGGT inspired confidence has systematic artifacts.}
        \label{fig:uncertainty_badpixelwise}
\end{figure}
\newpage

\end{document}